\documentclass[twoside,11pt]{article}

\usepackage[preprint,abbrvbib]{jmlr2e}
\usepackage{lastpage}
\usepackage{microtype}
\usepackage{subcaption}
\usepackage{booktabs}
\usepackage{amsmath,amsfonts}
\usepackage{mathtools}
\usepackage{array}
\usepackage{enumitem}
\usepackage{algorithm}
\usepackage{algorithmic}
\usepackage{placeins}
\hypersetup{
  hidelinks,
  pdftitle={Causal Retention in Interactive Agents: Interface Factorization and Selective Adaptation},
  pdfauthor={Shengjun Zhang, Tingyi Liu, Dong Xie, Yunlong Dong, Xiang Wang, and Cheng Zeng},
  pdfkeywords={causal retention, interactive agents, mechanism memory, causal adaptation, structural causal models}
}

\newcolumntype{L}[1]{>{\raggedright\arraybackslash}p{#1}}
\newcommand{\E}{\mathbb E}
\newcommand{\Pp}{\mathbb P}
\setlist[itemize]{leftmargin=1.55em,itemsep=0.16em,topsep=0.24em}
\setlist[enumerate]{leftmargin=1.7em,itemsep=0.16em,topsep=0.24em}

\jmlrheading{Preprint}{2026}{1--\pageref{LastPage}}{September 2026}{}{}{Shengjun Zhang, Tingyi Liu, Dong Xie, Yunlong Dong, Xiang Wang, and Cheng Zeng}
\ShortHeadings{Causal Retention in Interactive Agents}{Zhang, Liu, Xie, Dong, Wang, and Zeng}
\firstpageno{1}

\begin{document}

\title{Causal Retention in Interactive Agents: Interface Factorization and Selective Adaptation}

\author{\name Shengjun Zhang$^{1,2}$ \email sj.zhang@hubu.edu.cn\\
\name Tingyi Liu$^{3}$ \email tingyiL@whu.edu.cn\\
\name Dong Xie$^{4,*}$ \email xiedong04@baidu.com\\
\name Yunlong Dong$^{5}$ \email yunlongdong@outlook.com\\
\name Xiang Wang$^{1,2,*}$ \email wangxiang.whu@whu.edu.cn\\
\name Cheng Zeng$^{1,2}$ \email zc@hubu.edu.cn\\[2pt]
\addr $^1$School of Artificial Intelligence, Hubei University\\
$^2$Key Laboratory of Intelligent Sensing System and Security (Hubei University), Ministry of Education\\
$^3$School of Economics and Management, Wuhan University\\
$^4$Baidu Inc.\\
$^5$Independent Researcher\\
$^*$Correspondence: xiedong04@baidu.com; wangxiang.whu@whu.edu.cn}

\maketitle

\begin{abstract}%
Task performance need not determine which intervention mechanism an agent
retains. We study causal retention: whether a frozen learned state answers a
mechanism-probe map fixed independently of training, including action, context,
direct target, value, and delay. For finite structural causal model classes,
the optimal probe error is a Bayes decision risk. It vanishes exactly when
every learning-interface fiber lies within one probe-answer fiber; any state
obtained by post-processing that interface inherits the same lower bound. A
posterior-coverage theorem characterizes budgeted retesting, while an exact
edit decomposition shows that the shifted set is the unique support of an
error-free target update. Causal Core implements these conditions through
evidence-gated writing, readout filtering, temporal credit, hidden-context
setup, and local diagnostic updates. Experiments cover finite causal systems,
continuous simulators, an official TD-MPC2 world model, and
Qwen2.5-7B-Instruct. A frozen Qwen last-layer probe reaches 0.958 balanced
accuracy on source mechanisms but 0.583 on changed delays; the gated mechanism
state reaches 1.000 and accepts only 0.056 of synchronized-readout candidates.
In TD-MPC2, five
target states per actuator recover effect-sign accuracy from 0.057 to 0.948
without degrading stable responses. Causal retention is therefore distinct
from task sufficiency and source-domain decodability.
\end{abstract}

\begin{keywords}
causal retention, interactive agents, mechanism memory, causal adaptation, structural causal models
\end{keywords}

\section{Introduction and Related Work}
\label{sec:intro}

Training objectives identify equivalence classes of environments, not complete
intervention mechanisms. Two interactive systems may induce the same reward
process, preferred action, or one-step prediction target while disagreeing
about the direct target of an action, its active context, or its delay. A state
that is sufficient for the training task can therefore be insufficient for an
intervention query that the objective never had to answer.

The ability of a frozen learned state to answer such queries is called
\emph{causal retention}. The evaluator fixes the probe semantics independently
of the learned state, freezes that state after source learning, and then asks
which action changes which target, under which context, to what value, and
after what delay. Under a local mechanism shift, the same state must identify
the changed entry without overwriting invariant ones. This is a property of
what learning retained, not only of the behavior it produced.

Several causal-learning traditions address neighboring but different
estimands. Structural causal models define interventions and counterfactuals
\citep{pearl2009causality,peters2017elements}. Treatment-effect methods
estimate average, conditional, or individual effects on a designated outcome
\citep{rubin1974estimating,imbens2015causal,chernozhukov2018double,wager2018causal,kunzel2019metalearners,nie2021quasi,shalit2017estimating}.
Causal discovery instead estimates a graph or an interventional equivalence
class
\citep{spirtes2000causation,chickering2002optimal,hauser2012characterization,shimizu2006linear,zheng2018dags}.
A correct outcome effect or graph can remain silent about the action-indexed
value, delay, and context that must be edited after transfer.

Invariant prediction seeks predictors stable across environments
\citep{peters2016invariant,arjovsky2019invariant}. Causal representation
learning instead studies recovery of latent causal variables or graphs from
indirect observations
\citep{locatello2019challenging,scholkopf2021toward,gamella2025sanity,varici2025score}.
Intervention extrapolation uses an identifiable representation to predict the
effect of unseen actions on an outcome \citep{saengkyongam2024identifying}.
Causal retention fixes the downstream answer map and asks whether a completed
learning process preserved it. It neither requires recovery of the latent SCM
nor reduces the answer to one designated outcome.

Reinforcement learning, world models, and language agents optimize return,
transition prediction, or action quality
\citep{sutton2018reinforcement,ha2018worldmodels,hafner2023mastering,hansen2024tdmpc2,yao2023react}.
Causal reinforcement learning and data-fusion methods improve decisions under
interventions \citep{bareinboim2016causal,zhang2021causal}, while ranking-style
dynamics losses such as Hybrid$^2$ supervise which intervention should be
preferred \citep{zou2024hybrid2}. These objectives can be useful and exactly
optimized without requiring the frozen state to expose every coordinate of the
mechanism needed by a later probe. The distinction studied here is therefore
between an objective's behavioral sufficiency and a learned state's causal
sufficiency.

Concurrent work on Explicit Symbolic Behavioral Models uses executable
``mechanism memory,'' adaptive questions, and active world-model branches to
train editable symbolic policies \citep{shindo2026explicit}. There, mechanism
memory predicts symbolic events and rewards and is optimized jointly with
policy-specific questions. Here, the probe map and its direct-target semantics
are fixed independently of the learner, scoring occurs after the state is
frozen, and the main object is the smallest probe error permitted by a given
learning interface. The two formulations therefore use similar language for
different statistical questions.

One realization of causal retention is a frozen, queryable mechanism memory
with entries of the form
\[
  \theta=(a,c,y,v,\delta),
\]
where $a$ is an action, $c$ a context predicate, $y$ the direct target, $v$
the effect value, and $\delta$ the delay. After learning, the memory is frozen
and queried on held-out mechanism probes. Equal reward laws, rankings, graphs,
or action choices may still induce different probe answers.

For finite structural causal model (SCM) classes, the optimal population probe
error is an explicit Bayes risk. It is zero exactly when the interface quotient
refines the probe quotient, is monotone under refinement, and lower-bounds every
decoder of a post-processed state. The same construction with a metric loss
covers continuous intervention responses. A posterior-coverage identity
characterizes restricted target-domain retesting, and an exact edit
decomposition characterizes selective adaptation. Standard concentration and
testing inequalities are used only for finite-sample rates.

Causal Core is the corresponding evidence-gated memory layer, with target
admission, readout filtering, temporal credit, hidden-context setup, diagnostic
localization, and entrywise updating. Experiments compare policy, ranking,
conditional-discovery, tuple, latent-world-model, neural-probe, transfer,
language-model proposal, readout-unsafe, and ungated baselines. Continuous
responses are scored over held-out state, action-direction, and horizon
distributions. An official pretrained TD-MPC2 checkpoint supplies a public
implicit-world-model comparison. Frozen Qwen hidden states are evaluated with
linear and nonlinear decoders under externally fixed probe semantics.

\begin{figure}[t]
\centering
\includegraphics[width=0.98\linewidth]{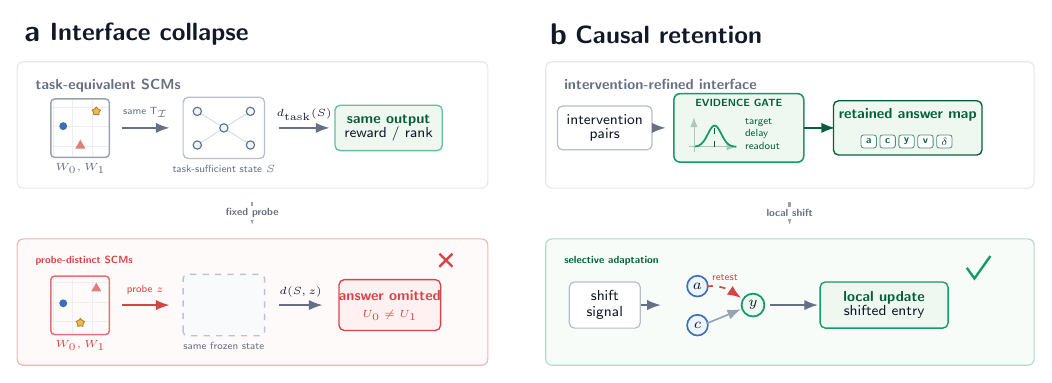}
\caption{Interface collapse and causal retention. A task interface can merge
probe-distinct SCMs. A refined interface retains the answer map and supports
selective adaptation.}
\label{fig:overview}
\end{figure}

\section{Problem Formulation}
\label{sec:prelim}

Let an interactive structural causal model (SCM) be
\[
  X_{t+1}=F(X_t,H_t,A_t,\varepsilon_{t+1}),\qquad
  O_t=\psi(X_t,H_t,\eta_t),
\]
with physical state $X_t\in\mathcal X$, optional latent context
$H_t\in\mathcal H$, action $A_t\in\mathcal A$, and observation $O_t$. Write
$\Xi_t=(X_t,H_t)\in\mathcal X\!\times\!\mathcal H$ for the full evaluator
state; the learner need not observe $H_t$. A protocol $\Pi$
maps histories to actions and induces a trajectory law
$P_M^\Pi$ for each SCM $M$.

For an action intervention at time $t$, write
$P_M^{do(A_t=a),\Pi}$ for the law obtained by replacing the protocol action
with $a$ at $t$ and leaving the remaining rollout policy fixed. A context
predicate is a map $c:\mathcal X\times\mathcal H\to\{0,1\}$. From the
equations of $M$, the evaluator fixes a structural index
\[
  \sigma_M(\xi,a)
  \in
  \bigl(\mathcal C\times\mathcal Y\times\{0,\ldots,D\}\bigr)
  \cup\{\bot\}.
\]
When non-null, $\sigma_M(\xi,a)=(c,y,\delta)$ records the active context
class, the variable directly acted on by the mechanism, and its first
structural response delay. A deterministic descendant or sensor readout is
not a direct target. In finite worlds $\sigma_M$ is read from the structural
assignment; in continuous simulators it is defined by the paired-pulse
functional in Appendix~\ref{app:world-model-probes}.

For $z=(\xi,a)$ and $r=(c,y,v,\delta)$, define
\[
  p_M(r\mid z)
  :=
  \mathbf 1\{\sigma_M(\xi,a)=(c,y,\delta)\}
  \mathbb P_M^{do(A_t=a),\Pi}
  \{X_{t+\delta}^y=v\mid \Xi_t=\xi\}.
\]
Let $\mathcal R$ be the finite set of candidate quadruples $(c,y,v,\delta)$.
If $\sigma_M(\xi,a)=\bot$, set $G_M(z)=\emptyset$; otherwise define
\[
  G_M(z)
  =
  \operatorname*{arg\,max}_{r\in\mathcal R}p_M(r\mid z),
\]
restricted to a probe class on which the maximizer is unique. In deterministic
SCMs a tuple $\theta=(a,c,y,v,\delta)$ is correct exactly when
$p_M((c,y,v,\delta)\mid z)=1$. In stochastic SCMs uniqueness is enforced by a
margin $\kappa>0$:
\[
  p_M(r^\star\mid z)
  -
  \max_{r\neq r^\star}p_M(r\mid z)
  \ge \kappa .
\]

\subsection{Mechanism Space}
Let
\[
  \Theta=\mathcal A\times\mathcal C\times\mathcal Y
  \times\mathcal V\times\{0,\ldots,D\}.
\]
A mechanism memory is a finite set $\mu\subset\Theta$. For a query
$z=(\xi,a)$,
define
\[
  T_\mu(\xi,a)
  =
  \{(c,y,v,\delta): (a,c,y,v,\delta)\in\mu,\ c(\xi)=1\}.
\]
Two answers are equivalent, written
$T_\mu(\xi,a)\equiv T_{\mu'}(\xi,a)$, when they agree on the active context,
direct target, value, and delay classes fixed by the evaluator. Exact tuple
evaluation uses ordinary set equality.

The constructive and empirical results use a representable mechanism schema:
for every evaluated SCM $M$, there is a finite ground-truth memory
$\mu_M^\star\subset\Theta$ such that
$T_{\mu_M^\star}(z)=G_M(z)$ for $Q$-almost every probe $z$. The lower bounds
are stated directly for $G_M$ and do not require such a representation.

\subsection{Mechanism Risk}
For probe distribution $Q$ and ground-truth memory $\mu^\star=\mu_M^\star$, the exact
mechanism risk is
\[
  R_Q(\widehat\mu,\mu^\star)
  =
  \mathbb P_{(\xi,a)\sim Q}\!\left[
    T_{\widehat\mu}(\xi,a)\not\equiv T_{\mu^\star}(\xi,a)
  \right].
\]
Graph error, reward regret, top-action accuracy, and prediction error are
not substitutes for $R_Q$: each marginalizes at least one coordinate of
$(a,c,y,v,\delta)$.

Coordinate risks localize the source of an error. Let
$\mathcal K=\{\mathrm{ctx},\mathrm{tar},\mathrm{val},\mathrm{del}\}$ and
let $\pi_k T_\mu(\xi,a)$ be the projection of the answer set onto coordinate
$k$. Define
\[
  R_Q^k(\widehat\mu,\mu^\star)
  =
  \mathbb P_Q\!\left[
  \pi_kT_{\widehat\mu}(\xi,a)\neq \pi_kT_{\mu^\star}(\xi,a)
  \right].
\]
When equivalence means equality of all four projections,
\[
  \max_{k\in\mathcal K} R_Q^k
  \le
  R_Q
  \le
  \sum_{k\in\mathcal K} R_Q^k .
\]
Thus the exact risk is the main target, while coordinate risks diagnose
which part of the mechanism failed. Two additional operational risks are
\[
\begin{aligned}
  R_{\mathrm{fp}}
  &=
  \mathbb P_Q[T_{\widehat\mu}(\xi,a)\neq\emptyset,\ T_{\mu^\star}(\xi,a)=\emptyset],
  \\
  R_{\mathrm{fn}}
  &=
  \mathbb P_Q[T_{\widehat\mu}(\xi,a)=\emptyset,\ T_{\mu^\star}(\xi,a)\neq\emptyset].
\end{aligned}
\]
They distinguish spurious mechanisms from missed mechanisms.

The fixed-memory protocol is essential. Exploration may use rewards,
observations, auxiliary descriptions, or rankings; evaluation fixes $\widehat\mu$ and
queries only the answer map $T_{\widehat\mu}$. Thus a behaviorally competent
learner receives no credit unless its stored mechanism answers the probe:
write $\widehat T_i=T_{\widehat\mu}(\xi_i,a_i)$ and
$T_i^\star=T_{\mu^\star}(\xi_i,a_i)$. Then
\[
  \widehat\mu\mapsto
  \{T_{\widehat\mu}(\xi_i,a_i)\}_{i=1}^m ,
\]
\[
  \mathrm{score}
  =
  1-\frac1m\sum_{i=1}^m
  \ell_i,
  \quad
  \ell_i=
  \mathbf 1\{\widehat T_i\not\equiv T_i^\star\}.
\]

\subsection{Selective Adaptation}
Let $\mu^s=\{\theta_j^s\}_{j=1}^n$ be a source memory. A schema map
$\phi=(\phi_A,\phi_C,\phi_Y)$ induces mapped entries
\[
  \phi(\theta_j^s)=
  (\phi_A(a_j^s),\phi_C(c_j^s),\phi_Y(y_j^s),v_j^s,\delta_j^s).
\]
The target memory is $\mu^t=\{\theta_j^t\}_{j=1}^n$. Define the stable and
shifted index sets
\[
  S=\{j:\theta_j^t=\phi(\theta_j^s)\},\qquad
  B=[n]\setminus S .
\]
For an adapted memory $\widehat\mu^t$,
\[
\begin{aligned}
  L_{\mathrm{adapt}}
  &=
  \frac1n\sum_{j=1}^n
  \mathbf 1\{\widehat\theta_j^t\neq \theta_j^t\},\\
  L_{\mathrm{stab}}
  &=
  \frac1{|S|}\sum_{j\in S}
  \mathbf 1\{\widehat\theta_j^t\neq \theta_j^t\},
\end{aligned}
\]
and $L_{\mathrm{shift}}$ is defined analogously over $B$; an empty-set loss
is defined as zero. Selective
adaptation requires $L_{\mathrm{stab}}=L_{\mathrm{shift}}=0$ under a small
target budget.
For exploration and sampling randomness conditional on a fixed source-target
pair, define
\[
  \mathcal R_{\mathrm{adapt}}
  =\mathbb E[L_{\mathrm{adapt}}\mid\mu^s,\mu^t],
  \qquad
  \mathcal R_{\mathrm{stab}}
  =\mathbb E[L_{\mathrm{stab}}\mid\mu^s,\mu^t],
\]
\[
  \mathcal R_{\mathrm{shift}}
  =\mathbb E[L_{\mathrm{shift}}\mid\mu^s,\mu^t].
\]
The expectation decomposes as
\[
  \mathcal R_{\mathrm{adapt}}
  =
  \frac{|S|}{n}\mathcal R_{\mathrm{stab}}
  +
  \frac{|B|}{n}\mathcal R_{\mathrm{shift}}.
\]
The operational failure probability is
\[
  \mathcal R_{\mathrm{op}}
  =
  \mathbb P(L_{\mathrm{stab}}>0\ \vee\ L_{\mathrm{shift}}>0),
\]
because a single missed shifted mechanism can be invisible in average loss
when $|B|\ll n$.

Two limiting cases illustrate the requirement. Pure transfer keeps
$\widehat\mu^t=\phi(\mu^s)$ and therefore fails on $B$. Broad re-learning may
fix $B$ but can corrupt $S$. The desired operation is an entrywise map
\[
  \widehat\theta_j^t
  =
  \begin{cases}
    \phi(\theta_j^s), & j\in S,\\
    \theta_j^t, & j\in B,
  \end{cases}
\]
with evidence deciding membership in $B$.

\subsection{Standing Conditions}

Each lower bound uses only the assumptions in its statement. The constructive
guarantees for Causal Core use the following grouped conditions.
\begin{enumerate}[label=(C\arabic*),ref=C\arabic*]
  \item \label{cond:alignment}
  The source and target memories contain $n$ aligned keys. The evaluator's
  probe distribution $Q$ is fixed independently of target-domain samples. The
  schema map $\phi$ is supplied or estimated from a disjoint alignment sample;
  its error is represented by the event $F_\phi$ rather than assumed away.
  Target-coordinate values and visible context predicates in the memory schema
  are measurable from the learner's observation history. Latent predicates are
  never supplied to the learner and are handled only through assigned setup
  and control regimes.
  \item \label{cond:interventions}
  A gate is estimated from independent, bounded outcomes collected under
  paired action and reference interventions at the same pre-intervention
  context. Equivalently, an observational implementation must satisfy
  consistency, conditional exchangeability, and positivity and use the
  corresponding propensity correction. An adaptive randomized policy may use
  the logged inverse-propensity contrast in
  Lemma~\ref{lem:adaptive-contrast}. The finite agents execute interventions
  and are scored descriptively; the continuous and simulator protocols use
  explicit paired pulses.
  \item \label{cond:margins}
  Every finite decision used by a gate is separated from its threshold. Write
  $\gamma_g>0$ for target, context, readout, and temporal margins,
  $\Delta=p-q>0$ for diagnostic separation, and $\eta_j>0$ for the excess
  loss of a wrong local-update candidate. The correct local candidate belongs
  to the finite class $\mathcal N_j$.
  \item \label{cond:metric}
  Metric-valued probe answers lie in a compact space and use a bounded,
  continuous distortion. The probe distribution is fixed before fitting. A
  discrete implementation may additionally fix a measurable partition; its
  paired contrast routine fails with probability at most $\beta_M$.
\end{enumerate}
Conditions~\ref{cond:interventions}--\ref{cond:margins} are sufficient, not
necessary. When they fail, the quotient and decision lower bounds still
apply, but the finite-sample upper bounds need not.

\subsection{Population Interfaces}
An interface $\mathcal I$ induces a population descriptor
$\tau_{\mathcal I,\Pi}(M)=P_M^{\mathcal I,\Pi}$: the data law generated by
the protocol. Examples include reward trajectories, observations, rankings,
prediction errors, diagnostic signals,
or mechanism probes. For probe class $\mathcal Z$ with answer map $G_M(z)$,
write
\[
\begin{aligned}
  M_0\sim_{\mathcal I,\Pi}M_1
  &\Longleftrightarrow
  P_{M_0}^{\mathcal I,\Pi}=P_{M_1}^{\mathcal I,\Pi},\\
  M_0\sim_{\mathcal Z}M_1
  &\Longleftrightarrow
  G_{M_0}(z)=G_{M_1}(z)\quad \forall z\in\mathcal Z .
\end{aligned}
\]
The relevant condition is whether $\sim_{\mathcal I,\Pi}$ refines
$\sim_{\mathcal Z}$ for mechanism probes. If not, the interface is too
coarse.
For the finite analysis, fix full-support distributions $\nu$ on $\mathcal M$
and $Q$ on $\mathcal Z$. Let
\[
  W\sim\nu,\qquad Z\sim Q,\qquad
  \mathsf T_{\mathcal I}=\tau_{\mathcal I,\Pi}(W),\qquad
  U=G_W(Z),
\]
where $Z$ is independent of $W$. The \emph{causal-retention risk} of the
interface is the probe Bayes risk
\begin{equation}
  \mathfrak B_{\nu,Q}(\mathcal I,\Pi)
  =
  \mathbb E\!\left[
    1-\max_{u\in\mathcal U}
    \mathbb P(U=u\mid\mathsf T_{\mathcal I},Z)
  \right].
  \label{eq:retention-bayes-risk}
\end{equation}
It will equal the smallest error of any decoder that sees the complete
population interface. For $U=(U_1,\ldots,U_d)$, define
$\mathfrak B_j$ by replacing $U$ with $U_j$ in
Eq.~\eqref{eq:retention-bayes-risk}. These risks depend on the interface and
the fixed probe map rather than on a chosen decoder.
The population minimax interface risk is
\[
  \mathfrak R(\mathcal I,\mathcal Z)
  =
  \inf_{\widehat G}
  \sup_{M\in\mathcal M}
  \mathbf 1\{\widehat G(\tau_{\mathcal I,\Pi}(M))\neq G_M\}.
\]
Theorem~\ref{thm:retention-factorization} gives the exact operational condition:
$\mathfrak R(\mathcal I,\mathcal Z)=0$ only when the interface separates all
probe-distinct SCMs.

For a compact metric answer space $(\mathcal U,\rho)$, define the expected and
worst-case metric risks
\begin{equation}
\begin{aligned}
  \mathfrak B^\rho_{\nu,Q}(\mathcal I,\Pi)
  &=\inf_g\mathbb E\rho\bigl(U,g(\mathsf T_{\mathcal I},Z)\bigr),\\
  \mathfrak R^\rho(\mathcal I,\mathcal Z)
  &=\inf_g\sup_{M\in\mathcal M,\,z\in\mathcal Z}
    \rho\bigl(G_M(z),g(\tau_{\mathcal I,\Pi}(M),z)\bigr).
\end{aligned}
  \label{eq:metric-retention-risk}
\end{equation}
For Hamming loss, $\mathfrak B^\rho=\mathfrak B$. The metric version therefore
changes the loss geometry, not the semantics of retention.

\subsection{Exact Causal Retention}
A possibly random learned state $S$ has exact causal retention for
$\mathcal Z$ on $\mathcal M$ if there exists a single decoder $d$ such that
\[
  \mathbb P_M\!\left\{d(S,z)=G_M(z)
  \text{ for every }z\in\mathcal Z\right\}=1
  \qquad \forall M\in\mathcal M .
\]
The probe class and its answer semantics are fixed independently of the learned
state. This makes causal retention independent of the internal implementation:
a memory table, recurrent state, or external module qualifies only to the
extent that it supports the same decoder after learning has stopped.

For ranking interfaces, the observed object is a distribution over ordered
pairs
\[
  \mathcal D_{\mathrm{rank}}
  =
  \{(i,j): U_M(i)>U_M(j)\},
\]
where $U_M(i)$ is the intervention utility. For reward interfaces it is a
trajectory marginal of $\sum_t r_t$. For graph interfaces it is an edge set.
For diagnostic adaptation it is instead an indexed surprise channel
$\{X_{j,r}\}_{j,r}$ tied to transferred memory entries. If two SCMs have the
same interface law but different probe answers, no estimator using that
interface can guarantee zero mechanism risk.

\section{Causal Core}
\label{sec:method}

Causal Core maintains mechanism memory independently of the policy optimizer.
Its state consists of a candidate set $C_t$, readout set $\mathcal R_t$, temporal buffer $B_t$,
diagnostic surprise vector $S_t$, and memory $\mu_t$. It writes
$(a,c,y,v,\delta)$ only when intervention evidence supports every coordinate.
For a candidate $\theta=(a,c,y,v,\delta)$, define its memory key
$\kappa(\theta)=(a,c,y,\delta)$. Let $a^0$ be the reference intervention for
$a$. For a learner-visible context predicate $c$, let
\[
\begin{aligned}
  \mathcal D_t^1(a,c,\delta)
  &=\{s:s+\delta\le t,\ A_s=a,\ c(\Xi_s)=1\},\\
  \mathcal D_t^0(a,c,\delta)
  &=\{s:s+\delta\le t,\ A_s=a^0,\ c(\Xi_s)=1\}.
\end{aligned}
\]
Write $n_t^u=|\mathcal D_t^u|$ and
$Y_s^\theta=\mathbf 1\{X_{s+\delta}^y=v\}$. A candidate does not mature
until $n_t^0,n_t^1>0$. Its paired empirical frequencies are
\begin{equation}
\label{eq:empirical-frequency}
  \widehat p_t^u(\theta)
  =
  \frac{1}{n_t^u}
  \sum_{s\in\mathcal D_t^u(a,c,\delta)}Y_s^\theta,
  \qquad u\in\{0,1\}.
\end{equation}
Under Condition~\ref{cond:interventions}, their difference estimates the
interventional contrast between $do(a)$ and $do(a^0)$ rather than an
unadjusted action association. The target gate is
\begin{equation}
\label{eq:target-gate}
\begin{aligned}
  \widehat\Delta_t(\theta)
  &=
  \widehat p_t^1(\theta)-\widehat p_t^0(\theta),\\
  \Gamma_t^{\mathrm{tar}}(\theta)
  &=
  \mathbf 1\{n_t^0n_t^1>0,\ \widehat\Delta_t(\theta)\ge\lambda_t\}.
\end{aligned}
\end{equation}
The full write gate is
\begin{equation}
\label{eq:admitted-set}
\begin{aligned}
  G_t(\theta)
  &=
  \prod_{g\in\{\mathrm{tar},\mathrm{time},\mathrm{readout},\mathrm{ctx}\}}
  \Gamma_t^g(\theta),\\
  C_t^{\mathrm{adm}}&=\{\theta\in C_t:G_t(\theta)=1\}.
\end{aligned}
\end{equation}
The memory transition is the key-preserving operator
\begin{equation}
\label{eq:memory-write}
  \mathsf W(\mu,\theta)
  =
  \{\theta'\in\mu:\kappa(\theta')\ne\kappa(\theta)\}\cup\{\theta\}.
\end{equation}
If several candidates mature at the same time, they are applied in decreasing
$\widehat\Delta_t$ with a fixed lexicographic tie-breaker. Thus the memory
contains at most one value for each action-context-target-delay key.
Admission is determined by the four gates; planner design does not enter the
definition of mechanism memory.
For a latent context, the sets in Eq.~\eqref{eq:empirical-frequency} are
instead indexed by the learner-known assignment to the setup or control
regime; membership never uses $c(\Xi_s)$ itself. The resulting context gate is
defined in Section~\ref{sec:hidden-context-setup}.

\begin{remark}
Ground-truth mechanism labels are not supplied during training. Supervision is
created by controlled interventions whose contrasts identify a target,
context, value, or delay coordinate. One-step predictive self-supervision uses
$o_{t+1}$ as the target given $(o_t,a_t)$; here the retained target is the
intervention-specific tuple $(a,c,y,v,\delta)$ queried after learning.
\end{remark}

\subsection{Readout Filtering}
If $Y$ and $R$ are synchronized, the two SCMs
\[
  a\to Y,\ R:=Y
  \qquad\text{and}\qquad
  a\to R,\ Y:=R
\]
can produce the same observations. Causal Core therefore excludes readout
candidates from direct-target writing unless an intervention separates them.
Formally, for a candidate target $u$, define the direct-target set
\[
  \mathcal Y_t^{\mathrm{dir}}
  =
  \{u\in\mathcal Y:
  u\notin\mathcal R_t\ \text{or}\ \mathrm{Sep}_t(u)=1\},
\]
\[
  \mathrm{Sep}_t(u)
  =
  \mathbf 1\{\exists b:
  P(X_{t+1}^u\mid do(b))
  \neq
  P(X_{t+1}^{r(u)}\mid do(b))\},
\]
where $r(u)$ denotes the synchronized readout paired with $u$ when known.
For formula readouts, the finite implementation uses, for each candidate
coordinate $u$, a class $\mathcal F_u$ of formulas whose inputs exclude $u$.
It computes
\[
  \widehat A_t(f,u)
  =\frac1{m_t}\sum_{r=1}^{m_t}
   \mathbf 1\{f(X_r)=X_r^u\},
  \qquad
  u\in\mathcal R_t
  \Longleftrightarrow
  \max_{f\in\mathcal F_u}\widehat A_t(f,u)\ge c_{\mathrm{ro}}.
\]
Proposition~\ref{prop:readout-filter-bound} controls both directions of this
decision under a finite-class margin. A variable classified as a readout can
re-enter the direct-target set only through a separating intervention.
The readout gate is
\[
  \Gamma_t^{\mathrm{readout}}(a,c,y,v,\delta)
  =
  \mathbf 1\{y\in\mathcal Y_t^{\mathrm{dir}}\}.
\]

\subsection{Temporal Credit}
For delayed effects, Causal Core tests stability of the lag:
\[
  \widehat\delta(a,c,y,v)
  \in
  \arg\max_{0\le\delta\le D}
  \widehat\Delta_t(a,c,y,v,\delta).
\]
An entry is written only if the selected lag beats competing recent actions
by a margin.
The analyzed margin form is
\[
  \widehat\Delta_t(a,c,y,v,\delta)
  -
  \max_{(a',v',\delta')\neq(a,v,\delta)}
  \widehat\Delta_t(a',c,y,v',\delta')
  \ge \gamma_{\mathrm{time}},
\]
where every contrast uses only matured paired trials from
Eq.~\eqref{eq:empirical-frequency}. This prevents a waiting action or a
post-effect observation from being written as the cause.

\subsection{Diagnostic Adaptation}
After transfer, mechanism $j$ produces binary surprise samples
$X_{j,r}\in\{0,1\}$. Causal Core computes
\[
  \overline X_j(m)=\frac1m\sum_{r=1}^m X_{j,r}
\]
and retests entries above a threshold. The edit is local:
$\widehat\theta_j^t$ may change only for selected $j$.
For $\tau=(p+q)/2$,
\begin{equation}
\label{eq:diagnostic-selection}
  \widehat B_m=\{j:\overline X_j(m)\ge \tau\}.
\end{equation}
\begin{equation}
\label{eq:diagnostic-entry-update}
  \widehat\theta_j^t=
  \begin{cases}
    \phi(\theta_j^s), & j\notin \widehat B_m,\\
    \widetilde\theta_j, & j\in \widehat B_m.
  \end{cases}
\end{equation}
For a selected entry, the identifying retest gives an independent local sample
$D_j^{\mathrm{re}}=\{(Z_{j,r},U_{j,r})\}_{r=1}^{r_j}$, where $U_{j,r}$ is
the observed response statistic specified by the intervention protocol, and a finite candidate
class $\mathcal N_j$. Let $g_\theta$ be the probe-answer map predicted by
candidate $\theta$. Define
\begin{equation}
\label{eq:local-retest}
\begin{aligned}
  \widehat L_j(\theta)
  &=
  \frac1{r_j}
  \sum_{r=1}^{r_j}
  \mathbf 1\{g_\theta(Z_{j,r})\ne U_{j,r}\},\\
  \widetilde\theta_j
  &\in
  \arg\min_{\theta\in\mathcal N_j}\widehat L_j(\theta).
\end{aligned}
\end{equation}
The corresponding local edit is
\begin{equation}
\label{eq:local-edit}
  \mathsf U_j(\mu,\widetilde\theta_j)
  =
  \{\theta_\ell\in\mu:\ell\ne j\}\cup\{\widetilde\theta_j\}.
\end{equation}
All entries outside $\widehat B_m$ are held fixed. Under
Eq.~\eqref{eq:local-edit}, stable corruption can occur only through the
set-selection event $\widehat B_m\neq B$.

\subsection{Hidden-Context Setup}
\label{sec:hidden-context-setup}
A hidden gate is not written from passive frequency alone. The module
actively searches setup actions that increase eligibility, then tests a
visible-separator class $\mathcal G$. A hidden label survives only when a
controlled setup regime and its matched control have separated success rates
and no visible separator explains the pattern.
For candidate $i$, let $E_i$ be latent eligibility and $Z_i$ success. The
learner observes the assigned regime, not $E_i$. It seeks a setup policy
$\pi_i^{\mathrm{set}}$ and a matched control $\pi_i^0$ such that
\[
  \mathbb P_{\pi_i^{\mathrm{set}}}(E_i=1)\ge 1-\epsilon_{\mathrm{set}},
  \qquad
  \mathbb P_{\pi_i^0}(E_i=1)\le\epsilon_0,
\]
\[
  \mathbb E_{\pi_i^{\mathrm{set}}}Z_i
  -\mathbb E_{\pi_i^0}Z_i
  \ge 2\gamma .
\]
With $m_i^{\mathrm{set}},m_i^0>0$ outcomes assigned to the two regimes, the
context gate uses
\[
  \widehat d_i
  =\frac1{m_i^{\mathrm{set}}}\sum_{r\in\mathcal D_i^{\mathrm{set}}}Z_{i,r}
   -\frac1{m_i^0}\sum_{r\in\mathcal D_i^0}Z_{i,r},
  \qquad
  \Gamma_i^{\mathrm{ctx}}
  =\mathbf 1\{\widehat d_i\ge\gamma,\ \widehat g_i=\varnothing\},
\]
where $\widehat g_i$ is a visible separator found in $\mathcal G$. The label
is retained only when the contrast passes and no visible predicate explains
it. Lemma~\ref{lem:paired-contrast} controls the contrast error, while
Proposition~\ref{prop:setup-cost} converts repeated setup reachability into
action cost when setup completion is observable. If passive access gives
$\mathbb P(E_i=1)=q_i$, then even an oracle that marks eligible trials needs
$m/q_i$ attempts on average to collect $m$ eligible outcomes.

\begin{algorithm}[t]
\caption{Causal Core memory update and adaptation}
\label{alg:causal-core-main}
\begin{algorithmic}[1]
\STATE Initialize $\mu$, readout set $\mathcal R$, buffer $B$, scores $d_j=0$.
\FOR{each interaction step $t$}
  \STATE Observe $o_t$, execute $a_t$, and observe the next response.
  \STATE Update $\mathcal R$ using readout formulas and separating
  interventions.
  \STATE Add candidate effects to $B$ with action, context, target, value,
  and delay fields.
  \FOR{each buffered candidate $(a,c,y,v,\delta)$ whose delay has matured}
    \IF{$(a,c,y,v,\delta)\in C_t^{\mathrm{adm}}$ from Eq.~\eqref{eq:admitted-set}}
      \STATE Update $\mu$ by Eq.~\eqref{eq:memory-write}.
    \ENDIF
  \ENDFOR
  \STATE For transferred entries, accumulate diagnostic surprise
  $d_j\leftarrow d_j+X_{j,t}$.
\ENDFOR
\STATE Select shifted candidates by Eq.~\eqref{eq:diagnostic-selection}
or top-$s$ diagnostic scores.
\FOR{each selected mechanism index $j$}
  \STATE Retest entry $j$ and compute $\widetilde\theta_j$ by
  Eq.~\eqref{eq:local-retest}.
  \STATE Update $\mu$ by Eq.~\eqref{eq:local-edit}.
\ENDFOR
\RETURN fixed memory $\mu$.
\end{algorithmic}
\end{algorithm}

\section{Theoretical Results}
\label{sec:theory}

The results connect interface fibers, frozen-state probe error, and the number
of target-domain retests. Theorem~\ref{thm:retention-factorization} gives an
exact factorization criterion and the optimal probe risk.
Corollary~\ref{cor:finite-interface-lower} gives the finite-sample two-model
bound. Theorem~\ref{thm:list-adaptation-lower} gives the exact posterior
coverage of budgeted retesting, while
Propositions~\ref{prop:diagnostic-selection-main} and
\ref{prop:local-retest-main} give constructive rates under the stated margins.
Theorem~\ref{thm:end-to-end-risk-main} identifies the unique support of an
exact target update and composes the finite-sample events.
Theorem~\ref{thm:metric-retention} gives the metric-valued extension. The core
results follow by conditional decision minimization, fiber factorization, and
an exact edit identity. Standard testing inequalities enter only the
finite-sample corollaries. Auxiliary separations and consequences appear in
Appendix~\ref{app:auxiliary-statements}. Let
$M_0\sim_{\mathcal I,\Pi}M_1$ denote equality of the interface law under
protocol $\Pi$, and let $M_0\sim_{\mathcal Z}M_1$ denote equality of all
fixed probe answers on $\mathcal Z$.
All logarithms are natural.

\begin{theorem}
\label{thm:retention-factorization}
Let $\mathcal M$ and $\mathcal Z$ be finite, let $\nu$ and $Q$ have full
support, and use the variables in Eq.~\eqref{eq:retention-bayes-risk}. Then the
following statements hold.
\begin{enumerate}[label=(\roman*),leftmargin=1.7em]
  \item The probe Bayes risk has the exact representation
  \[
    \mathfrak B_{\nu,Q}(\mathcal I,\Pi)
    =\inf_g\mathbb P\{g(\mathsf T_{\mathcal I},Z)\ne U\}.
  \]
  The infimum is attained by a deterministic decoder.

  \item $\mathfrak B_{\nu,Q}(\mathcal I,\Pi)=0$ if and only if
  \[
    M_0\sim_{\mathcal I,\Pi}M_1
    \Longrightarrow
    M_0\sim_{\mathcal Z}M_1
    \qquad \forall M_0,M_1\in\mathcal M .
  \]
  Hence a zero-risk population decoder exists exactly when the interface
  quotient refines the probe quotient.

  \item If two interface-equivalent models disagree on a set of probes of
  $Q$-mass $\rho$, every population decoder has worst-case mechanism risk at
  least $\rho/2$.

  \item Let $D_N$ be a training transcript whose conditional law given $W$
  depends on $W$ only through $\mathsf T_{\mathcal I}$, and let $S$ be any
  randomized function of $D_N$. Every frozen-state decoder satisfies
  \[
    \mathbb P\{d(S,Z)\ne U\}
    \ge \mathfrak B_{\nu,Q}(\mathcal I,\Pi).
  \]

  \item If $\mathsf T_1=h(\mathsf T_2)$, then
  $\mathfrak B(\mathsf T_2)\le\mathfrak B(\mathsf T_1)$. For a vector answer,
  \begin{equation}
    \max_{j\in[d]}\mathfrak B_j
    \le \mathfrak B_{\nu,Q}(\mathcal I,\Pi)
    \le \sum_{j=1}^d\mathfrak B_j .
    \label{eq:bayes-coordinate-sandwich}
  \end{equation}
\end{enumerate}
\end{theorem}

\begin{proof}
Write $\mathsf T=\mathsf T_{\mathcal I}$. Conditional on
$(\mathsf T,Z)=(t,z)$, a decoder returning $u$ has error
$1-\mathbb P(U=u\mid t,z)$. Pointwise minimization over $u$, followed by
expectation over $(\mathsf T,Z)$, proves (i).

The risk in (i) is zero exactly when every conditional distribution
$P(U\mid\mathsf T=t,Z=z)$ is a point mass. Full support of $\nu$ and $Q$ turns
this into a pointwise statement. Thus, whenever
$\tau_{\mathcal I,\Pi}(M_0)=\tau_{\mathcal I,\Pi}(M_1)$,
\[
  G_{M_0}(z)
  =g(\tau_{\mathcal I,\Pi}(M_0),z)
  =g(\tau_{\mathcal I,\Pi}(M_1),z)
  =G_{M_1}(z)
\]
for every $z\in\mathcal Z$. This proves the forward implication in (ii).
Conversely, suppose the quotient implication holds. Every interface class
$C=\{M:\tau_{\mathcal I,\Pi}(M)=t\}$ has a unique probe map $g_C$, so define
$g(t,z)=g_C(z)$. Then $U=g(\mathsf T,Z)$ and the Bayes risk is zero.

For (iii), let
$A=\{z:G_{M_0}(z)\ne G_{M_1}(z)\}$ with $Q(A)=\rho$. Under the common
interface descriptor, any randomized decoder has the same output law in the
two models. Pointwise on $A$,
\[
  \mathbf 1\{\widehat G(z)\ne G_{M_0}(z)\}
  +
  \mathbf 1\{\widehat G(z)\ne G_{M_1}(z)\}
  \ge1.
\]
Expectation over decoder randomness and $z\sim Q$ gives
\[
  R_Q(\widehat G,M_0)+R_Q(\widehat G,M_1)\ge\rho,
\]
and hence $\max_iR_Q(\widehat G,M_i)\ge\rho/2$.

For (iv), the transcript condition makes $S$ a randomized garbling of
$\mathsf T$. Composing $d(S,Z)$ with the garbling kernel produces a randomized
decoder from $(\mathsf T,Z)$. Randomization cannot improve the pointwise
minimum in (i), proving the bound.

For (v), every decoder based on $\mathsf T_1=h(\mathsf T_2)$ is also a decoder
based on $\mathsf T_2$, so the latter infimum is no larger. Any joint decoder
induces a decoder for coordinate $j$; therefore its vector error is at least
$\mathfrak B_j$, proving the lower bound in
Eq.~\eqref{eq:bayes-coordinate-sandwich}. Conversely, concatenate the
coordinatewise Bayes decoders. Its vector error is contained in the union of
their coordinate errors, and the union bound proves the upper inequality.
\end{proof}

For a uniform two-model fiber and a point-mass probe on which the answers
differ, the probe Bayes risk is $1/2$. A probe-revealing interface has risk
zero.

With finite samples, total variation controls the residual indistinguishability
of two probe-distinct models.

\begin{corollary}
\label{cor:finite-interface-lower}
Let $P_0^N$ and $P_1^N$ be the $N$-sample transcript laws of two SCMs whose
probe answers differ on a set of $Q$-mass $\rho$. Then every estimator obeys
\[
  \max_{i\in\{0,1\}}R_Q(\widehat G,M_i)
  \ge
  \frac{\rho}{2}
  \left(1-\lVert P_0^N-P_1^N\rVert_{\mathrm{TV}}\right).
\]
If the samples are conditionally independent with one-sample laws $P_0,P_1$
and
$d=\min\{\mathrm{KL}(P_0\Vert P_1),\mathrm{KL}(P_1\Vert P_0)\}$, then
\[
  \max_iR_Q(\widehat G,M_i)
  \ge
  \frac{\rho}{2}
  \left[1-\sqrt{\frac{Nd}{2}}\right]_+ .
\]
\end{corollary}

\begin{proof}
Fix a probe $z$ on which the answers $g_0(z)$ and $g_1(z)$ differ. Let
$A_i(z)=\{\widehat G(z)=g_i(z)\}$. The events $A_0(z)$ and $A_1(z)$ are
disjoint, hence
\[
\begin{aligned}
 P_0^N(A_0(z)^c)+P_1^N(A_1(z)^c)
 &\ge P_0^N(A_1(z))+P_1^N(A_1(z)^c)\\
 &=1-\{P_1^N(A_1(z))-P_0^N(A_1(z))\}\\
 &\ge1-\lVert P_0^N-P_1^N\rVert_{\mathrm{TV}}.
\end{aligned}
\]
Integrating over the disagreement set and dividing the sum of the two risks by
two proves the first inequality. For product laws,
$\mathrm{KL}(P_0^N\Vert P_1^N)=N\mathrm{KL}(P_0\Vert P_1)$. Pinsker's
inequality \citep{tsybakov2009introduction}, applied in the better of the two
directions, gives
$\lVert P_0^N-P_1^N\rVert_{\mathrm{TV}}\le\sqrt{Nd/2}$; truncation at zero
gives the second display.
\end{proof}

The proof is decision-theoretic: it does not pass through entropy, mutual
information, or a decoder-class restriction.

Risk decompositions, standard-target lower bounds, and explicit two-world
constructions for ranking, readout, and hidden-context interfaces appear in
Appendix~\ref{app:auxiliary-statements}.

\begin{theorem}
\label{thm:list-adaptation-lower}
Let $\mathcal B_s=\{b\subseteq[n]:|b|=s\}$, let $B\in\mathcal B_s$, and let
$Z$ be any pre-retest diagnostic signal. For $s\le k<n$, define
\begin{equation}
  \Gamma_k(z)
  =
  \max_{L\subseteq[n]:\,|L|\le k}
  \sum_{b\in\mathcal B_s:\,b\subseteq L}
  \mathbb P(B=b\mid Z=z).
  \label{eq:posterior-coverage}
\end{equation}
Then the largest probability that a $Z$-measurable retest list contains every
shifted entry is
\begin{equation}
  \sup_{\mathcal L:\,|\mathcal L(Z)|\le k}
  \mathbb P\{B\subseteq\mathcal L(Z)\}
  =\mathbb E\Gamma_k(Z).
  \label{eq:optimal-retest-coverage}
\end{equation}
Internal randomization cannot increase this value. Zero omission is possible
if and only if, for almost every $z$,
\begin{equation}
  \left|\bigcup\{b:\mathbb P(B=b\mid Z=z)>0\}\right|\le k.
  \label{eq:retest-support-condition}
\end{equation}
If $B$ is uniform on $\mathcal B_s$ and independent of $Z$, the optimal
coverage is exactly $\binom{k}{s}/\binom{n}{s}$. Finally, if
$Z_1=h(Z_2)$, then $\mathbb E\Gamma_k(Z_1)\le\mathbb E\Gamma_k(Z_2)$.
\end{theorem}

\begin{proof}
For a deterministic list $L(z)$, conditional coverage is
\[
  \mathbb P\{B\subseteq L(z)\mid Z=z\}
  =\sum_{b\subseteq L(z)}\mathbb P(B=b\mid Z=z).
\]
Maximizing this finite sum separately for every $z$ gives
Eq.~\eqref{eq:posterior-coverage}; averaging gives
Eq.~\eqref{eq:optimal-retest-coverage}. A randomized rule is a mixture of
deterministic lists, so its conditional coverage is a convex combination of
values no larger than $\Gamma_k(z)$.

The equality $\Gamma_k(z)=1$ holds precisely when one list of size at most
$k$ contains every $b$ with positive posterior mass. Such a list exists
precisely when the union of those posterior-support sets has size at most
$k$, proving Eq.~\eqref{eq:retest-support-condition}. Under the uniform
independent law, a list of size $\ell$ contains exactly $\binom{\ell}{s}$ of
the $\binom ns$ possible shifted sets. This is maximized at $\ell=k$.
Finally, every list rule based on $Z_1=h(Z_2)$ is also available to a rule
based on $Z_2$; taking suprema in
Eq.~\eqref{eq:optimal-retest-coverage} proves monotonicity.
\end{proof}

\begin{proposition}
\label{prop:diagnostic-selection-main}
For each transferred entry $j\in[n]$, suppose the diagnostic samples
$X_{j,1},\ldots,X_{j,m}\in[0,1]$ are independent. Suppose also that
$\mathbb E X_{j,r}\ge p$ for shifted entries $j\in B$ and
$\mathbb E X_{j,r}\le q$ for stable entries $j\notin B$, with
$\Delta=p-q>0$. Let
$\tau=(p+q)/2$ and
\[
  \overline X_j=\frac1m\sum_{r=1}^m X_{j,r},
  \qquad
  \widehat B=\{j:\overline X_j\ge\tau\}.
\]
If $m\ge 2\Delta^{-2}\log(n/\delta)$, then
$\mathbb P(\widehat B=B)\ge1-\delta$.
Moreover, in the symmetric one-shift Bernoulli family with
$p=1/2+\Delta/2$, $q=1/2-\Delta/2$, and $0<\Delta\le1/2$, every estimator
with localization error at most $\varepsilon$ requires
\[
  m\ge
  \frac{(1-\varepsilon)\log n-\log2}{8\Delta^2}.
\]
Thus the $\Delta^{-2}\log n$ scaling is necessary and sufficient up to
constants.
\end{proposition}

\begin{proof}
For $j\in B$,
\[
\begin{aligned}
  \mathbb P\{\overline X_j<\tau\}
  \le
  \exp\{-2m(p-\tau)^2\}
  =
  \exp\{-m\Delta^2/2\}.
\end{aligned}
\]
For $j\notin B$,
\[
\begin{aligned}
  \mathbb P\{\overline X_j\ge\tau\}
  \le
  \exp\{-2m(\tau-q)^2\}
  =
  \exp\{-m\Delta^2/2\}.
\end{aligned}
\]
These two applications of Hoeffding's inequality
\citep{hoeffding1963probability}, followed by a union bound over the $n$
entries, give
$\mathbb P(\widehat B\ne B)\le n\exp\{-m\Delta^2/2\}\le\delta$. This is the
selection guarantee used by Eq.~\eqref{eq:diagnostic-selection}; the
top-$s$ variant is proved in Appendix~\ref{app:proof-top-s-main}.
The matching lower-bound calculation is given in
Appendix~\ref{app:diagnostic-lower}.
\end{proof}

Theorem~\ref{thm:list-adaptation-lower} characterizes every retest list of size
$k$ without reducing the diagnostic signal to a scalar information budget.
In contrast,
Proposition~\ref{prop:diagnostic-selection-main} recovers the shifted set with
$m=O(\Delta^{-2}\log(n/\delta))$ diagnostic samples per entry.

\begin{proposition}
\label{prop:local-retest-main}
Fix a selected entry $j$ and a finite candidate class $\mathcal N_j$. Let
$\theta_j^\star$ be the target-domain mechanism and define the population
retest loss
\[
  L_j(\theta)=
  \mathbb E\,\mathbf 1\{g_\theta(Z_j)\ne U_j\},
  \qquad \theta\in\mathcal N_j .
\]
Suppose $\theta_j^\star\in\mathcal N_j$ and
$L_j(\theta)-L_j(\theta_j^\star)\ge\eta_j>0$ for every
$\theta\ne\theta_j^\star$, and suppose the $r_j$ retest pairs are independent
and identically distributed. If the retest sample size is
\[
  r_j\ge \frac{2}{\eta_j^2}\log\frac{|\mathcal N_j|}{\delta_j},
\]
then the empirical minimizer in Eq.~\eqref{eq:local-retest} satisfies
$\mathbb P(\widetilde\theta_j=\theta_j^\star)\ge1-\delta_j$.
\end{proposition}

\begin{proof}
For a wrong candidate $\theta$, set
\[
  D_\theta
  =
  \mathbf 1\{g_\theta(Z_j)\ne U_j\}
  -
  \mathbf 1\{g_{\theta_j^\star}(Z_j)\ne U_j\}.
\]
Then $D_\theta\in[-1,1]$ and
$\mathbb E D_\theta=L_j(\theta)-L_j(\theta_j^\star)\ge\eta_j$. If
$\widehat L_j(\theta)\le\widehat L_j(\theta_j^\star)$, then the empirical
mean $\overline D_\theta$ is nonpositive. For
$A_\theta=\{\widehat L_j(\theta)\le\widehat L_j(\theta_j^\star)\}$,
Hoeffding's inequality \citep{hoeffding1963probability} gives
\[
  \mathbb P(A_\theta)
  \le
  \mathbb P\{\overline D_\theta-\mathbb E D_\theta\le-\eta_j\}
  \le
  e^{-r_j\eta_j^2/2}.
\]
A union bound over $\mathcal N_j\setminus\{\theta_j^\star\}$ yields
$\mathbb P(\widetilde\theta_j\ne\theta_j^\star)
\le |\mathcal N_j|e^{-r_j\eta_j^2/2}\le\delta_j$.
\end{proof}

\begin{theorem}
\label{thm:end-to-end-risk-main}
Fix an aligned source-target pair satisfying
Condition~\ref{cond:alignment}. Write
$m_j=\phi(\theta_j^s)$ and
$B=\{j:m_j\ne\theta_j^t\}$, with $S=[n]\setminus B$. For any selected set
$D\subseteq[n]$ and local estimates $\{\widetilde\theta_j:j\in D\}$, let
\[
  \widehat\theta_j^t(D)
  =
  \begin{cases}
    m_j, & j\notin D,\\
    \widetilde\theta_j, & j\in D.
  \end{cases}
\]
Then
\begin{align}
  nL_{\mathrm{adapt}}
  &=|B\setminus D|
    +\sum_{j\in D}\mathbf 1\{\widetilde\theta_j\ne\theta_j^t\},
    \label{eq:exact-edit-decomposition}\\
  |S|L_{\mathrm{stab}}
  &=\sum_{j\in D\cap S}
    \mathbf 1\{\widetilde\theta_j\ne m_j\},
    \label{eq:stable-edit-decomposition}\\
  |B|L_{\mathrm{shift}}
  &=|B\setminus D|+
    \sum_{j\in D\cap B}
    \mathbf 1\{\widetilde\theta_j\ne\theta_j^t\}.
    \label{eq:shift-edit-decomposition}
\end{align}
Consequently the adapted memory is exact if and only if
$B\subseteq D$ and $\widetilde\theta_j=\theta_j^t$ for every $j\in D$.
Moreover, for any exact target memory $\bar\theta=\theta^t$,
\begin{equation}
  \{j:\bar\theta_j\ne m_j\}=B;
  \label{eq:unique-edit-support}
\end{equation}
thus $B$ is the unique realized edit support of an exact update relative to
the mapped source memory.

Let $F_{\mathrm{src}}$ be the event that the frozen source memory has a
missing entry or a wrong target, context, value, or delay, including a target
error caused by readout contamination. Let $F_\phi$ be schema-map failure,
$F_{\mathrm{diag}}=\{\widehat B\ne B\}$, and $F_{\mathrm{upd}}$ be failure
of any local retest on $B$. If their probabilities are bounded by
$\delta_{\mathrm{src}},\delta_\phi,\delta_{\mathrm{diag}}$, and
$\delta_{\mathrm{upd}}$, then the update in
Eq.~\eqref{eq:diagnostic-entry-update} satisfies
\[
  \mathbb P(L_{\mathrm{adapt}}>0)
  \le
  \delta_{\mathrm{src}}+\delta_\phi
  +\delta_{\mathrm{diag}}+\delta_{\mathrm{upd}}.
\]
On the complement of these events,
$L_{\mathrm{stab}}=L_{\mathrm{shift}}=L_{\mathrm{adapt}}=0$. For every
realization,
\[
  \mathbb E L_{\mathrm{adapt}}
  =
  \frac{|S|}{n}\mathbb E L_{\mathrm{stab}}
  +
  \frac{|B|}{n}\mathbb E L_{\mathrm{shift}} .
\]
\end{theorem}

\begin{proof}
For $j\notin D$, the update retains $m_j$, and therefore
\[
  \mathbf 1\{\widehat\theta_j^t(D)\ne\theta_j^t\}
  =\mathbf 1\{j\in B\}.
\]
For $j\in D$, the same error indicator is
$\mathbf 1\{\widetilde\theta_j\ne\theta_j^t\}$. Summing the first identity
over $j\notin D$ and the second over $j\in D$ proves
Eq.~\eqref{eq:exact-edit-decomposition}. On $S$, $m_j=\theta_j^t$; on $B$,
$m_j\ne\theta_j^t$. Splitting the two sums accordingly proves
Eqs.~\eqref{eq:stable-edit-decomposition} and
\eqref{eq:shift-edit-decomposition}.

Every term on the right of Eq.~\eqref{eq:exact-edit-decomposition} is
nonnegative. It vanishes precisely when $B\setminus D=\varnothing$ and every
selected estimate equals its target value, giving the stated necessary and
sufficient condition. If $\bar\theta=\theta^t$, then
$\bar\theta_j\ne m_j$ holds precisely when
$\theta_j^t\ne m_j$, which is the definition of $j\in B$. This proves
Eq.~\eqref{eq:unique-edit-support}.

Now let
$E=(F_{\mathrm{src}}\cup F_\phi\cup F_{\mathrm{diag}}\cup
F_{\mathrm{upd}})^c$. On $E$, the mapped source entries are correct,
$D=\widehat B=B$, and
$\widetilde\theta_j=\theta_j^t$ for every selected $j$. The exact condition
just proved gives $L_{\mathrm{adapt}}=0$. Therefore
\[
  \{L_{\mathrm{adapt}}>0\}
  \subseteq
  F_{\mathrm{src}}\cup F_\phi\cup
  F_{\mathrm{diag}}\cup F_{\mathrm{upd}},
\]
and the probability bound follows by the union bound. The decomposition of
the mean loss follows from partitioning the exact coordinate errors:
\[
\begin{aligned}
  L_{\mathrm{adapt}}
  &=
  \frac1n\sum_{j\in S}
    \mathbf 1\{\widehat\theta_j^t\ne\theta_j^t\}
  +\frac1n\sum_{j\in B}
    \mathbf 1\{\widehat\theta_j^t\ne\theta_j^t\}\\
  &=
  \frac{|S|}{n}L_{\mathrm{stab}}
  +
  \frac{|B|}{n}L_{\mathrm{shift}} .
\end{aligned}
\]
Taking expectation gives the displayed identity.
Appendix~\ref{app:proof-end-to-end-risk} gives the coordinate-level loss
expansion and decomposes $F_{\mathrm{src}}$ into its target, context,
readout, and temporal gates.
\end{proof}

Substitution of the target-contrast, readout, hidden-setup, temporal,
diagnostic, schema, and local-update bounds into
Theorem~\ref{thm:end-to-end-risk-main} is given in
Appendix~\ref{app:proof-end-to-end-sample-main}. The corresponding top-$s$
diagnostic result appears in Appendix~\ref{app:proof-top-s-main}; its selected
empirical means recover the shifted set under the same
$\Delta^{-2}\log(n/\delta)$ scaling.

\begin{theorem}
\label{thm:metric-retention}
Let $\mathcal M$ and $\mathcal Z$ be finite, let $\nu$ and $Q$ have full
support, and let $(\mathcal U,\rho)$ be a compact metric space with continuous
$\rho$. Use Eq.~\eqref{eq:metric-retention-risk} and write
$V=(\mathsf T_{\mathcal I},Z)$. Then:
\begin{enumerate}[label=(\roman*),leftmargin=1.7em]
  \item The expected metric risk has the pointwise representation
  \[
    \mathfrak B^\rho_{\nu,Q}(\mathcal I,\Pi)
    =\mathbb E\!\left[
      \min_{a\in\mathcal U}
      \mathbb E\{\rho(U,a)\mid V\}
    \right],
  \]
  and a deterministic decoder attains the minimum.

  \item $\mathfrak B^\rho_{\nu,Q}=0$ and
  $\mathfrak R^\rho=0$ if and only if every interface fiber has one probe
  answer map; equivalently,
  $M_0\sim_{\mathcal I,\Pi}M_1$ implies
  $M_0\sim_{\mathcal Z}M_1$.

  \item If $D_N$ and $S$ satisfy the transcript condition in
  Theorem~\ref{thm:retention-factorization}, every frozen-state decoder obeys
  \[
    \mathbb E\rho\{U,d(S,Z)\}
    \ge\mathfrak B^\rho_{\nu,Q}(\mathcal I,\Pi).
  \]
  If $\mathsf T_1=h(\mathsf T_2)$, then
  $\mathfrak B^\rho(\mathsf T_2)\le
  \mathfrak B^\rho(\mathsf T_1)$ and
  $\mathfrak R^\rho(\mathsf T_2)\le
  \mathfrak R^\rho(\mathsf T_1)$.

  \item If two interface-equivalent models $M_0,M_1$ have answer maps
  $G_0,G_1$, every possibly randomized decoder of their common interface
  satisfies
  \[
    \max_{i\in\{0,1\}}
    \mathbb E_{Z\sim Q}
      \rho\{\widehat G(Z),G_i(Z)\}
    \ge
    \frac12\mathbb E_{Z\sim Q}\rho\{G_0(Z),G_1(Z)\}.
    \label{eq:metric-two-fiber-bound}
  \]
  In particular, separation by at least $2\varepsilon$ on a set of
  $Q$-mass $q$ forces worst-case metric risk at least $q\varepsilon$.
\end{enumerate}
\end{theorem}

\begin{proof}
Condition on $V=v$. A deterministic answer $a$ has conditional loss
$\mathbb E\{\rho(U,a)\mid V=v\}$. Compactness and continuity give a minimizer;
pointwise minimization and averaging prove (i). Randomization only averages
these conditional losses and cannot improve their minimum.

The risk in (i) is zero exactly when $U$ is constant conditional on every
$(\mathsf T_{\mathcal I},Z)$ value with positive probability. Full support and
the metric property turn this into equality of the complete probe maps inside
each interface fiber. The same condition is plainly necessary and sufficient
for zero worst-case risk, proving (ii).

For (iii), composing a decoder of $S$ with the random kernel from
$\mathsf T_{\mathcal I}$ to $S$ gives a randomized decoder of the population
interface. Part (i) lower-bounds its loss. If $\mathsf T_1=h(\mathsf T_2)$,
every decoder based on $\mathsf T_1$ is available from $\mathsf T_2$; taking
the two infima proves both monotonicity statements.

For (iv), interface equivalence makes the output law of $\widehat G(z)$ the
same under $M_0$ and $M_1$. The triangle inequality gives, pointwise in $z$
and in decoder randomness,
\[
  \rho\{G_0(z),G_1(z)\}
  \le
  \rho\{G_0(z),\widehat G(z)\}
  +\rho\{\widehat G(z),G_1(z)\}.
\]
Expectation over $Z$ and decoder randomness shows that the sum of the two
risks is at least
$\mathbb E_Q\rho\{G_0(Z),G_1(Z)\}$. Their maximum is at least half this
quantity, which proves Eq.~\eqref{eq:metric-two-fiber-bound}. The final
statement follows by integrating the assumed separation over its probe set.
\end{proof}

\section{Experiments}
\label{sec:experiments}

The evaluation covers finite SCM families, readout and hidden-context variants,
few-shot renamed adaptation, metric-valued control responses, language-model
hidden states, and language-model proposal policies. Every learned state is frozen before the
mechanism probes are scored. Exact protocols, seeds, and auxiliary tables are in
Appendices~\ref{app:finite-scm-results}--\ref{app:llm-evaluation}.

Proposition~\ref{prop:weak-interfaces} supplies exact indistinguishable model
pairs. The empirical runs measure the corresponding fixed-probe failures and
the finite-sample behavior of the constructive rules. Renamed transfer uses
shifted-set localization and entrywise retesting. Continuous responses are
integrated over held-out state, action-direction, and horizon distributions;
language-model results use a separately fixed probe map.

Held-out state-action probes query target, value, delay, context, and stable
preservation. Reward and top-action accuracy are not terms in the score. For
probe set $\mathcal P$,
\begin{equation}
  \mathrm{MechAcc}
  =
  \frac1{|\mathcal P|}
  \sum_{(\xi,a)\in\mathcal P}
  \mathbf 1\{T_{\widehat\mu}(\xi,a)\equiv T_{\mu^\star}(\xi,a)\}.
\label{eq:mechanism-accuracy}
\end{equation}
Thus $\widehat R_{\mathrm{mech}}=1-\mathrm{MechAcc}$ estimates $R_Q$. The
coordinate estimates are
\begin{equation}
  \widehat R_k
  =
  \frac1{|\mathcal P|}
  \sum_{(\xi,a)\in\mathcal P}
  \mathbf 1\{\pi_kT_{\widehat\mu}(\xi,a)\neq\pi_kT_{\mu^\star}(\xi,a)\},
\label{eq:empirical-coordinate-risk}
\end{equation}
with false-write risks obtained by replacing the indicator by the
corresponding empty-set event. Few-shot renamed adaptation reports
\begin{equation}
  \mathrm{AdaptScore}
  =
  \mathbf 1\{L_{\mathrm{stab}}=0,\ L_{\mathrm{shift}}=0\},
\label{eq:adaptscore}
\end{equation}
so a method must both preserve stable entries and correct shifted ones.
The coordinate risks are empirical counterparts of $\mathfrak B_j$; they
identify which answer coordinate remains unresolved in the frozen state.

The finite discovery and language-model tables additionally report the direct
target projection
\begin{equation}
  E(\mu)=\{(a,y):\exists(c,v,\delta),
  (a,c,y,v,\delta)\in\mu\},
  \label{eq:target-edge-projection}
\end{equation}
using set precision, recall, and F1 against $E(\mu^\star)$. This projected F1
is accompanied by context recall and false-write counts; it is not presented
as the exact tuple accuracy in Eq.~\eqref{eq:mechanism-accuracy}.
The reported false-edge count is
$|E(\widehat\mu)\setminus E(\mu^\star)|$; readout false edges are the subset
whose target coordinate is a designated deterministic proxy.

\begin{figure}[t]
\centering
\includegraphics[width=\linewidth]{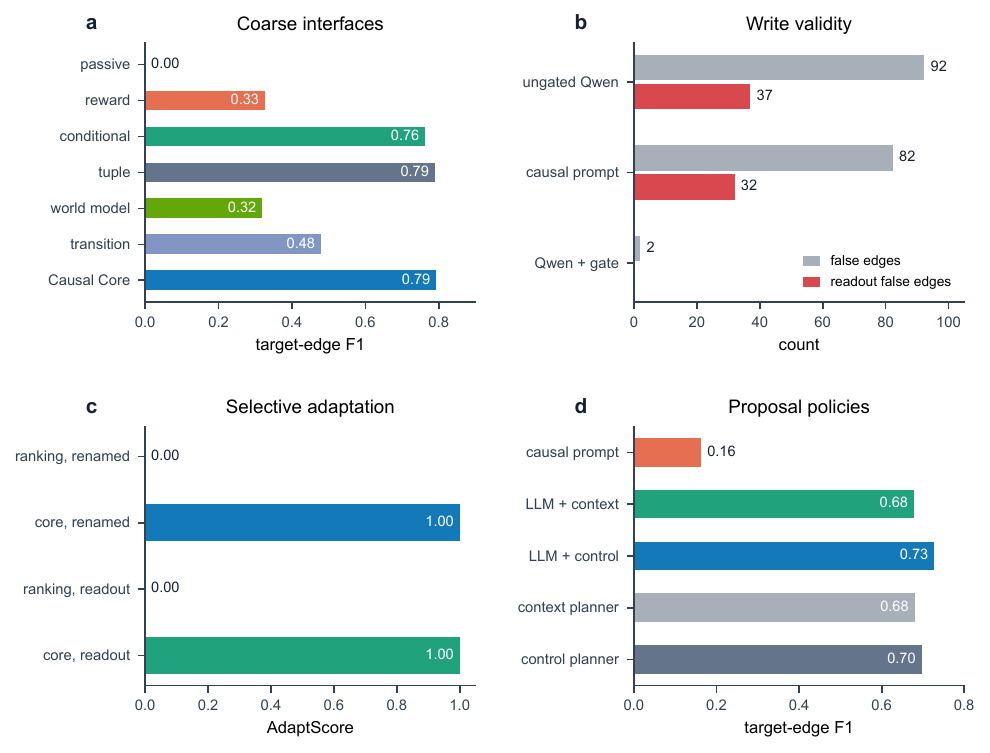}
\caption{Fixed-probe target and adaptation outcomes.}
\label{fig:main-results-v10}
\end{figure}

\subsection{Protocols and Baselines}
The suite contains finite worlds, readout and hidden-context variants,
continuous paired interventions, and language-model proposal policies.
Baselines match the available interface: passive and reward learners, ranking
loss, conditional discovery, latent world models, tuple and neural probes, the
official TD-MPC2 world model, ungated proposals, and transfer ablations. All are
evaluated after freezing, using
Eqs.~\eqref{eq:mechanism-accuracy}--\eqref{eq:target-edge-projection} or the
metric loss in Eq.~\eqref{eq:metric-retention-risk}.

\subsection{Coarse Interfaces}
The complex noisy-hidden family realizes positive probe Bayes risk under
Theorem~\ref{thm:retention-factorization}. Passive
correlation gives target-edge F1 $0.000$, while a reward-optimized learner has
target-edge F1 $0.326$. A ranking-loss predictor receives supervised
intervention-order labels; nevertheless, in the noisy-hidden ranking task it
has target-edge F1 $0.000$ even though Causal Core reaches $0.764$. Stronger
trajectory-matched baselines control for output format. On the expanded
matched run, a thresholded
conditional-discovery baseline has target-edge F1 $0.64$, and its readout-oracle variant
reaches $0.76$ but still writes hidden labels with false-positive count $1.50$.
The readout-oracle tuple writer has target-edge F1 $0.79$ but $6.19$ hidden-label false
positives. A latent world model reaches next-bit accuracy $0.905$ yet has
target-edge F1 $0.24$; with a readout oracle it reaches $0.32$ and writes
$39.88$ hidden-label false positives. Causal Core has target-edge F1 $0.79$
with zero readout and hidden-label false positives. High task or predictive
performance therefore coexists with poor fixed-probe recovery.

Under the uniform prior, each exact two-world separation in
Proposition~\ref{prop:weak-interfaces} has probe Bayes risk $1/2$ and is
decoder-independent. The neural rows support a narrower empirical statement:
high transition accuracy does not make the mechanism answer recoverable by the
specified frozen-state decoders, including readout-oracle variants. The
hidden-state experiment below adds linear and nonlinear decoders of a public
language model rather than inferring omission from its generated answer alone.

\subsection{Write Validity and Local Adaptation}
Indiscriminate writing raises recall at the cost of validity. The ungated Qwen
explorer writes $92.3$ false entries
and $37.0$ readout false positives; the context-search and control-planner
gates reduce these counts to $6.0/0.0$ and $2.0/0.0$, respectively. In renamed
and semantic-readout adaptation, non-diagnostic and
readout-unsafe baselines have AdaptScore $0$, while Causal Core has
AdaptScore $1.00$. The evaluation uses the diagnostic selection and local
update in Eqs.~\eqref{eq:diagnostic-selection}--\eqref{eq:local-edit}, whose
finite-sample guarantees are Propositions~\ref{prop:diagnostic-selection-main}
and \ref{prop:local-retest-main}.

\subsection{Hidden Contexts and Continuous Responses}
Hidden gates test the setup part of the interface. Active setup improves
hidden recall from $0.125$ after 80 interactions to $0.625$ after 160
interactions at zero hidden false positives in the four-family horizon sweep.
Rare eligibility events leave some hidden contexts outside the observed
quotient. In a continuous metric SCM, global
linear regression reaches F1 $0.667$ because it writes a deterministic readout
and extends a context-specific effect to the wrong cell; metric Causal Core
has F1 $1.000$ with no readout false positives.

We next evaluate the official $5.2$M-parameter TD-MPC2 checkpoint
\citep{hansen2024tdmpc2} on cheetah-run, walker-walk, and reacher-easy. For
saved state $x$, unit action direction $u$, and horizon
$h\in\{1,2,4\}$, the evaluator records the continuous response
\begin{equation}
 \Delta_h(x,u)
 =\frac{O_h(x,+\epsilon u)-O_h(x,-\epsilon u)}{2\epsilon},
 \qquad \epsilon=0.25 .
 \label{eq:distributional-response}
\end{equation}
Linear and two-layer nonlinear decoders receive the frozen source latent state,
latent intervention difference, $u$, and $h$. Thus a failure is not tied to a
linear readout. The nonlinear decoder has source NRMSE $0.629$ and effect-sign
accuracy $0.919$ over nine task-seed runs.

At transfer, one actuator changes sign and its index varies across runs. The
frozen model's shifted-actuator sign accuracy falls to $0.057$. A global
probe-head update using the target pairs raises it to $0.694$ but increases
stable-response NRMSE from $0.673$ to $1.460$. The selective update compares
the two actuator hypotheses on five target states per actuator, localizes the
shift in all nine runs, and reaches sign accuracy $0.948$. Its random-direction
NRMSE is $0.630$, equal to the oracle actuator map within reported precision,
while its stable-response NRMSE remains $0.673$. Appendix~\ref{app:world-model-probes}
gives per-task results, decoder training, and diagnostic-margin sensitivity.

\begin{figure}[t]
\centering
\includegraphics[width=\linewidth]{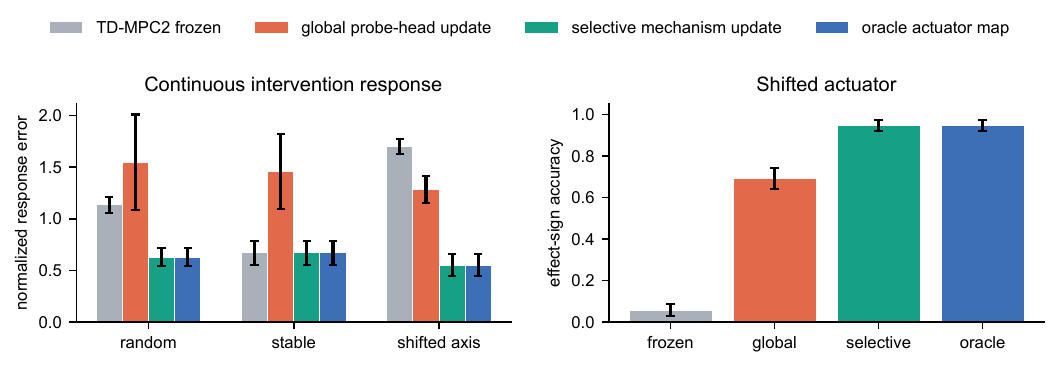}
\caption{Distributional TD-MPC2 responses under one actuator remapping.}
\label{fig:tdmpc2-retention}
\end{figure}

\subsection{Language-Model States and Proposal Policies}
A frozen-state experiment tests whether mechanism answers are retained inside
Qwen2.5-7B-Instruct \citep{yang2024qwen2}, rather than judging generated
explanations alone. Logistic probes are fit at five layers on 12 procedural
families and then frozen for six disjoint families. A two-layer decoder and an
RBF decoder test nonlinear accessibility at the final layer. Action and state
names are independently replaced by neutral tokens in every family and
condition.

The final-layer linear decoder reaches balanced accuracy $0.958$ on source
mechanisms and $0.847$ after renaming, but $0.583$ on the changed-delay subset.
The nonlinear RBF decoder obtains $0.944$, $0.875$, and $0.250$, respectively.
Thus the source result cannot be dismissed as complete absence of mechanism
information, while that information does not support the held-out changed
delay. A paired-evidence rule recovers every changed delay but accepts $0.944$
of synchronized readouts. Applying the admissible-write gate preserves
changed-delay accuracy $1.000$ and lowers readout FPR to $0.056$.
Table~\ref{tab:llm-hidden-results} gives all controls.

\begin{figure}[t]
\centering
\includegraphics[width=\linewidth]{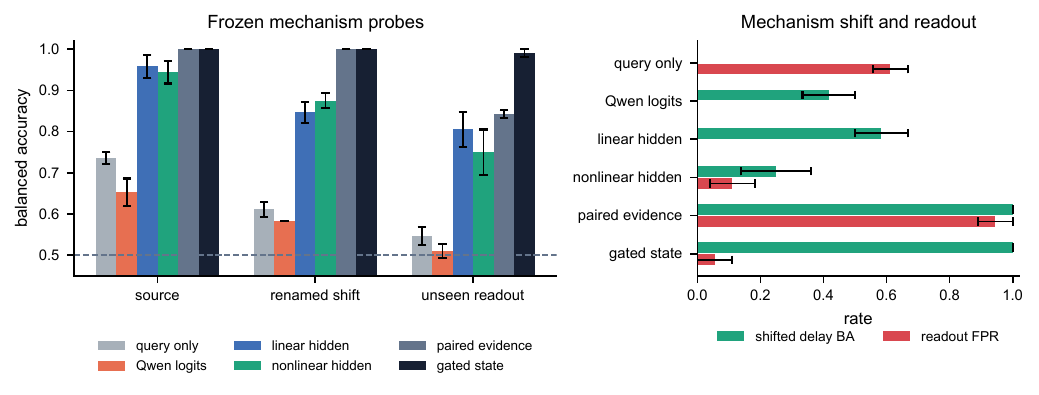}
\caption{Qwen hidden-state retention on held-out mechanism families.}
\label{fig:llm-hidden-retention}
\end{figure}

In the proposal-policy experiment, the same Qwen model chooses legal
interventions and the evidence gate processes the resulting transitions.
Causal prompting alone has target-edge F1 $0.163$ and $32.0$ readout false
positives. Gating raises target-edge F1 to $0.678$--$0.727$ and eliminates
readout false positives. Table~\ref{tab:llm-proposal-results} gives the full
breakdown.

Figure~\ref{fig:main-results-v10} localizes four distinct errors in the frozen
answer map: ranking loses target identity, ungated writing loses validity,
source dynamics loses the shifted sign, and prompting loses readout control.
These errors persist under relabeling, hidden gates, or readout corruption even
when the corresponding task objective remains useful.

\subsection{Reproducibility}
Code for generating the SCM families, running the baselines, reproducing the
MuJoCo and language-model experiments, and rebuilding every figure is
available at \url{https://github.com/ShengjunZhang/mechanism-memory}. The
source-data tables and raw result files are included with the submission.
Appendix~\ref{app:experimental-protocols} gives the random seeds, budgets,
model variants, probe definitions, and implementation details used in the
reported results.

\FloatBarrier
\section{Conclusion}
\label{sec:conclusion}

Causal retention asks whether a frozen learned state preserves an
interventional answer map beyond what task success requires. Probe Bayes risk,
posterior retest coverage, and the exact edit decomposition connect interface
identification to selective adaptation without assuming a particular decoder.
Finite SCMs, continuous control, a pretrained world model, and frozen Qwen
states expose target, context, value, and delay errors despite useful task or
source-decoding performance. The continuous results concern a fixed
distribution of state, intervention, and horizon queries rather than recovery
of an unrestricted latent SCM.

\appendix

\section{Auxiliary Theoretical Results}
\label{app:auxiliary-statements}

The finite constructions below establish the interface separations and the
extensions used by the selective-adaptation analysis.
All logarithms are natural.

\subsection{Coarse Behavioral and Predictive Interfaces}

\begin{proposition}
\label{prop:weak-interfaces}
There are finite deterministic structural causal model classes in which
reward, optimal-policy, intervention-ranking, graph, synchronized-readout,
and passive hidden-context interfaces are strictly coarser than the exact
mechanism-probe interface. For each interface, a point-mass probe distribution
has minimax mechanism risk at least $1/2$.
\end{proposition}

\begin{proof}
Each construction consists of two models with the same coarse interface and
different answers to one mechanism probe. The lower bound then follows from
Theorem~\ref{thm:retention-factorization} with $\rho=1$.

For reward and optimal-policy interfaces, let $U,Y\in\{0,1\}$ and let $a$ be
the only nontrivial action. In $M_0$, $a$ directly sets $Y\leftarrow1$ and
leaves $U$ unchanged. In $M_1$, $a$ directly sets $U\leftarrow1$ and the
structural equation is $Y\leftarrow U$. Initialize $U=Y=0$ and let reward be
$Y$. Every policy induces the same reward sequence in the two models, and the
optimal value and optimal-policy set agree. The direct target of $a$ is $Y$ in
$M_0$ and $U$ in $M_1$.

For intervention ranking, add a null action $b$ with utility zero. In both
models above, $a$ yields utility one and $b$ yields utility zero. Hence every
pairwise ranking label is identical, while the direct target of the preferred
action still differs.

For a graph interface, use the same graph $A\to Y$ in both models and let
$do(A=1)$ set $Y\leftarrow1$ in $M_0$ and $Y\leftarrow0$ in $M_1$ from a
common baseline $Y=1/2$. The graph is identical, whereas the value coordinate
of the mechanism answer differs. A delay or context coordinate can be varied
in the same way without changing the graph.

For a synchronized-readout interface, let $R$ be an observed proxy. In $M_0$,
$a$ directly sets $Y\leftarrow1$ and $R\leftarrow Y$; in $M_1$, $a$ directly
sets $R\leftarrow1$ and $Y\leftarrow R$. Both produce the observed pair
$(Y,R)=(1,1)$ after $a$, but the direct target is $Y$ in $M_0$ and $R$ in
$M_1$.

For passive hidden-context observation, fix $p\in(0,1)$. In $M_0$, action $a$
sets $Y\leftarrow1$ with probability $p$ independently on each trial. In
$M_1$, a latent $H\sim\mathrm{Bernoulli}(p)$ is drawn before each trial and
$a$ deterministically sets $Y\leftarrow1$ if and only if $H=1$. The passive
law of $(a,Y)$ is Bernoulli with parameter $p$ in both models, while only
$M_1$ has a hidden eligibility context. These pairs establish all claimed
strict coarsenings.
\end{proof}

The proposition does not assert that reward, graph estimation, or ranking is
uninformative. It states that these population objects need not determine all
coordinates of the mechanism answer. Additional intervention contrasts can
refine the interface and remove a particular ambiguity.

\subsection{Finite-Sample Contrast Gates}
\label{app:contrast-gates}

The constructive results use paired intervention contrasts rather than raw
action correlations. The following calculation supplies the common
finite-sample term for the target, context, and temporal gates.

\begin{lemma}
\label{lem:paired-contrast}
For $u\in\{0,1\}$, let
$Y_1^u,\ldots,Y_{n_u}^u$ be independent random variables in $[0,1]$ with
means $p_u$. Define
\[
  \widehat\Delta
  =\frac1{n_1}\sum_{r=1}^{n_1}Y_r^1
   -\frac1{n_0}\sum_{r=1}^{n_0}Y_r^0,
  \qquad
  \Delta=p_1-p_0.
\]
Then, for every $x>0$,
\[
  \mathbb P\{|\widehat\Delta-\Delta|\ge x\}
  \le
  2\exp\!\left\{
  -\frac{2x^2}{n_1^{-1}+n_0^{-1}}
  \right\}.
\]
If a gate accepts when $\widehat\Delta\ge\lambda$ and every population
contrast is at least $\lambda+\gamma$ or at most $\lambda-\gamma$, the
probability of any error among $K$ candidates is at most
\[
  2K\exp\!\left\{
  -\frac{2\gamma^2}{n_1^{-1}+n_0^{-1}}
  \right\}.
\]
For a balanced design $n_0=n_1=m$, this becomes $2K e^{-m\gamma^2}$.
\end{lemma}

\begin{proof}
Center all observations and write
\[
  \widehat\Delta-\Delta
  =\sum_{r=1}^{n_1}\frac{Y_r^1-p_1}{n_1}
   -\sum_{r=1}^{n_0}\frac{Y_r^0-p_0}{n_0}.
\]
Each summand in the first sum has an interval of possible values of length
$1/n_1$, and each summand in the second has interval length $1/n_0$. Hence
the sum of squared interval lengths is
\[
  n_1\left(\frac1{n_1}\right)^2
  +n_0\left(\frac1{n_0}\right)^2
  =\frac1{n_1}+\frac1{n_0}.
\]
Hoeffding's inequality \citep{hoeffding1963probability} gives the stated
exponential bound for each tail; adding the two tails proves the first
display. A population contrast on either side of the threshold can be
misclassified only if
$|\widehat\Delta-\Delta|\ge\gamma$. Applying the first bound and then a union
bound over $K$ candidates proves the remaining claims.
\end{proof}

The independent-pair calculation extends to adaptive interaction when action
probabilities are logged.

\begin{lemma}
\label{lem:adaptive-contrast}
Let $(\mathcal F_t)_{t=0}^N$ be the interaction filtration. At round $t$, let
$Y_t(1),Y_t(0)\in[0,1]$ be potential responses and suppose
\[
  A_t\perp\!\!\!\perp (Y_t(1),Y_t(0))\mid\mathcal F_{t-1},
  \qquad
  e_t=\mathbb P(A_t=1\mid\mathcal F_{t-1})\in[\epsilon,1-\epsilon]
\]
almost surely for some $0<\epsilon\le1/2$. Observe
$Y_t=A_tY_t(1)+(1-A_t)Y_t(0)$ and define
\[
  \widehat\Delta_N
  =\frac1N\sum_{t=1}^N
  \left\{\frac{A_tY_t}{e_t}
  -\frac{(1-A_t)Y_t}{1-e_t}\right\},
\]
\[
  \overline\Delta_N
  =\frac1N\sum_{t=1}^N
  \mathbb E[Y_t(1)-Y_t(0)\mid\mathcal F_{t-1}].
\]
Then, for every $x>0$,
\[
  \mathbb P\{ |\widehat\Delta_N-\overline\Delta_N|\ge x\}
  \le2\exp\{-N\epsilon^2x^2/2\}.
\]
Consequently, $K$ adaptive gates whose average contrasts are separated from
their thresholds by $\gamma$ have joint error probability at most
$2K\exp\{-N\epsilon^2\gamma^2/2\}$.
\end{lemma}

\begin{proof}
Set
\[
  V_t=\frac{A_tY_t}{e_t}
      -\frac{(1-A_t)Y_t}{1-e_t},
  \qquad
  d_t=\mathbb E[Y_t(1)-Y_t(0)\mid\mathcal F_{t-1}].
\]
Consistency and sequential randomization imply
\[
\begin{aligned}
 \mathbb E[V_t\mid\mathcal F_{t-1}]
 &=\frac{\mathbb E[A_tY_t(1)\mid\mathcal F_{t-1}]}{e_t}
   -\frac{\mathbb E[(1-A_t)Y_t(0)\mid\mathcal F_{t-1}]}{1-e_t}\\
 &=\mathbb E[Y_t(1)\mid\mathcal F_{t-1}]
   -\mathbb E[Y_t(0)\mid\mathcal F_{t-1}]
 =d_t.
\end{aligned}
\]
Thus $D_t=V_t-d_t$ is a martingale difference. Conditional on
$\mathcal F_{t-1}$, $V_t$ lies between $-1/(1-e_t)$ and $1/e_t$; subtracting
$d_t$ does not change the interval width, and
\[
  \frac1{e_t}+\frac1{1-e_t}\le\frac2\epsilon.
\]
Hoeffding's conditional lemma therefore gives
\[
  \mathbb E[e^{\lambda D_t}\mid\mathcal F_{t-1}]
  \le \exp\{\lambda^2/(2\epsilon^2)\}.
\]
Iterating conditional expectation yields
\[
  \mathbb E\exp\!\left\{\lambda\sum_{t=1}^ND_t\right\}
  \le\exp\{N\lambda^2/(2\epsilon^2)\}.
\]
For $\lambda>0$, Markov's inequality gives
\[
  \mathbb P\!\left\{\sum_{t=1}^ND_t\ge Nx\right\}
  \le\exp\{-\lambda Nx+N\lambda^2/(2\epsilon^2)\}.
\]
The minimizing value is $\lambda=\epsilon^2x$, which gives
$\exp\{-N\epsilon^2x^2/2\}$. Applying the same argument to $-D_t$ and adding
the tails proves the first claim. The gate bound follows by setting
$x=\gamma$ and taking a union bound over $K$ candidates.
\end{proof}

\begin{proposition}
\label{prop:readout-filter-bound}
Let $\mathcal F$ be a finite class of candidate readout formulas and let
$A(f)=\mathbb P\{f(X)=R\}$. From $m$ independent observations define
$\widehat A_m(f)=m^{-1}\sum_{r=1}^m\mathbf 1\{f(X_r)=R_r\}$. The rule declares
$R$ a formula readout when $\max_{f\in\mathcal F}\widehat A_m(f)\ge c$.
If some $f^\star$ satisfies $A(f^\star)\ge c+\gamma$, its false-negative
probability is at most $e^{-2m\gamma^2}$. If every $f$ satisfies
$A(f)\le c-\gamma$, the false-positive probability is at most
$|\mathcal F|e^{-2m\gamma^2}$.
\end{proposition}

\begin{proof}
In the first case, a false negative implies
$\widehat A_m(f^\star)<c$. Therefore
\[
\begin{aligned}
 \mathbb P\{\text{false negative}\}
 &\le
 \mathbb P\{\widehat A_m(f^\star)-A(f^\star)\le-\gamma\}\\
 &\le e^{-2m\gamma^2}.
\end{aligned}
\]
In the second case, a false positive implies that at least one
$f\in\mathcal F$ satisfies
$\widehat A_m(f)-A(f)\ge\gamma$. Hoeffding's inequality gives probability
at most $e^{-2m\gamma^2}$ for each fixed $f$; a union bound over
$\mathcal F$ proves the result.
\end{proof}

\subsection{Hidden-Context Acquisition}

The no-signal bound is exact for a uniformly hidden label. If
$B\sim\operatorname{Unif}(\{0,1\}^m)$ and $I(B;S)=0$, then $B$ remains uniform
conditional on $S$. Consequently every deterministic or randomized decoder
$\widehat B$ based on $S$ satisfies
\[
  \E d_{\mathrm H}(\widehat B,B)=m/2,
  \qquad
  \Pp\{\widehat B=B\}=2^{-m}.
\]
Active setup changes the rate at which informative eligible outcomes arrive.

\begin{proposition}
\label{prop:setup-cost}
Fix a hidden-gated candidate. Suppose a setup policy reaches an eligible state
within at most $h$ actions with probability at least $\beta$, independently
across reset attempts. To collect $m$ eligible process outcomes with
probability at least $1-\delta$, it is sufficient to run
\[
  N\ge\frac{8}{\beta}
  \left(m+\log\frac1\delta\right)
\]
setup attempts, for total action cost at most $N(h+1)$. Passive access is the
same calculation with the ambient eligibility probability $q$ in place of
$\beta$; its expected number of attempts is $m/q$.
\end{proposition}

\begin{proof}
Let $K_N$ be the number of successful setup attempts. It stochastically
dominates a $\mathrm{Binomial}(N,\beta)$ variable with mean
$\nu=N\beta$. The stated condition implies
\[
  \nu\ge8\left(m+\log\frac1\delta\right),
  \qquad m\le\nu/2.
\]
The multiplicative Chernoff lower-tail inequality
\citep{boucheron2013concentration} therefore gives
\[
\begin{aligned}
  \mathbb P(K_N<m)
  &\le\mathbb P(K_N\le\nu/2)
   \le e^{-\nu/8}\\
  &\le e^{-m-\log(1/\delta)}
   =e^{-m}\delta
   \le\delta.
\end{aligned}
\]
Each successful attempt contributes one process outcome and uses at most
$h+1$ actions. Under passive access, eligible arrivals are Bernoulli with
parameter $q$, so the waiting time for $m$ arrivals is negative binomial with
expectation $m/q$; the same Chernoff calculation gives the high-probability
$1/q$ factor.
\end{proof}

\subsection{Randomized Lists and Adaptive Stopping}
\label{app:proof-list-adaptation-lower}

Let $R$ be internal randomness independent of $(B,Z)$. Conditional on
$Z=z$, the coverage of $\mathcal L(z,R)$ is a convex combination of the
coverage values of deterministic lists and is therefore at most
$\Gamma_k(z)$. An adaptive procedure that stops after retesting at most $k$
distinct entries induces the list of all entries it examined. If its final
memory is exact, this list must contain every shifted entry unless an untested
entry was supplied by an external oracle. Hence its success probability is at
most $\mathbb E\Gamma_k(Z)$ as well. Adaptive ordering can reduce the number
of tests on favorable paths, but cannot increase the coverage permitted by a
hard cap of $k$ distinct entries.

\subsection{Diagnostic Localization Lower Bound}
\label{app:diagnostic-lower}

Consider the symmetric one-shift family in
Proposition~\ref{prop:diagnostic-selection-main}. Let $J$ be uniform on
$[n]$. Conditional on $J=j$, the $m$ samples at coordinate $j$ are
$\mathrm{Bernoulli}(1/2+\Delta/2)$ and those at every other coordinate are
$\mathrm{Bernoulli}(1/2-\Delta/2)$. Denote the joint law by $P_j$.

For $u=(1+\Delta)/2$ and $v=(1-\Delta)/2$,
\[
  \operatorname{kl}(u,v)
  =\Delta\log\frac{1+\Delta}{1-\Delta}.
\]
When $0<\Delta\le1/2$,
\[
  \log\frac{1+\Delta}{1-\Delta}
  =\log(1+\Delta)-\log(1-\Delta)
  \le\Delta+\frac{\Delta}{1-\Delta}
  \le3\Delta,
\]
so both directed Bernoulli divergences are at most $4\Delta^2$. For
$j\ne k$, $P_j$ and $P_k$ differ only at coordinates $j$ and $k$; product
additivity therefore gives
\[
  \mathrm{KL}(P_j\Vert P_k)
  \le8m\Delta^2.
\]
If $P_\circ=n^{-1}\sum_{k=1}^nP_k$, convexity in the second argument yields
\[
\begin{aligned}
 I(J;X)
 &=\frac1n\sum_{j=1}^n\mathrm{KL}(P_j\Vert P_\circ)\\
 &\le\frac1{n^2}\sum_{j,k=1}^n\mathrm{KL}(P_j\Vert P_k)
 \le8m\Delta^2.
\end{aligned}
\]
Fano's inequality \citep{cover2006elements} now implies
\[
  \mathbb P(\widehat J\ne J)
  \ge1-\frac{8m\Delta^2+\log2}{\log n}.
\]
Requiring this error to be at most $\varepsilon$ and rearranging proves
\[
  m\ge\frac{(1-\varepsilon)\log n-\log2}{8\Delta^2}.
\]

\subsection{\texorpdfstring{Top-$s$}{Top-s} Diagnostic Selection}
\label{app:proof-top-s-main}

Suppose $|B|=s$, shifted entries have means at least $p$, stable entries have
means at most $q<p$, and $\Delta=p-q$. Let $\widehat B_s$ contain the $s$
largest empirical means. If
\[
  \min_{j\in B}\overline X_j>\frac{p+q}{2}
  \quad\text{and}\quad
  \max_{j\notin B}\overline X_j<\frac{p+q}{2},
\]
then $\widehat B_s=B$. Hoeffding's inequality and a union bound give
\[
  \Pp\{\widehat B_s\ne B\}
  \le n\exp(-m\Delta^2/2).
\]
Consequently $m\ge2\Delta^{-2}\log(n/\delta)$ is sufficient for exact
top-$s$ recovery with probability at least $1-\delta$.

\subsection{Coordinate Loss and End-to-End Sample Composition}
\label{app:proof-end-to-end-risk}

Let $e_{j,k}$ indicate an error in coordinate
$k\in\{\mathrm{ctx},\mathrm{tar},\mathrm{val},\mathrm{del}\}$ of entry
$j$. Exact entry error satisfies
\[
  \max_k e_{j,k}
  \le \mathbf 1\{\widehat\theta_j^t\ne\theta_j^t\}
  \le \sum_k e_{j,k}.
\]
Summing first over $S$ and $B$ and then dividing by $n$ yields
\[
  L_{\mathrm{adapt}}
  =\frac{|S|}{n}L_{\mathrm{stab}}
   +\frac{|B|}{n}L_{\mathrm{shift}},
\]
and bounds each term by its coordinate losses. Let
$F_{\mathrm{tar}},F_{\mathrm{ctx}},F_{\mathrm{ro}},F_{\mathrm{time}}$, and
$F_{\mathrm{hid}}$ denote, respectively, a target-contrast error, a visible
context error, a readout-classification error, a delay error, and a hidden
setup or hidden-label error in the source memory. Then
\[
  F_{\mathrm{src}}
  \subseteq
  F_{\mathrm{tar}}\cup F_{\mathrm{ctx}}\cup F_{\mathrm{ro}}
  \cup F_{\mathrm{time}}\cup F_{\mathrm{hid}}.
\]
Diagnostic selection controls whether the appropriate target entry is tested;
local retesting controls its replacement coordinates. This is the component
factorization used by Theorem~\ref{thm:end-to-end-risk-main}.

\label{app:proof-end-to-end-sample-main}
For balanced paired designs, Lemma~\ref{lem:paired-contrast} gives
\[
\begin{aligned}
  \Pp(F_{\mathrm{tar}})&\le2K_{\mathrm{tar}}
     e^{-m_{\mathrm{tar}}\gamma_{\mathrm{tar}}^2},\\
  \Pp(F_{\mathrm{ctx}})&\le2K_{\mathrm{ctx}}
     e^{-m_{\mathrm{ctx}}\gamma_{\mathrm{ctx}}^2},\\
  \Pp(F_{\mathrm{time}})&\le2K_{\mathrm{time}}
     e^{-m_{\mathrm{time}}\gamma_{\mathrm{time}}^2}.
\end{aligned}
\]
Proposition~\ref{prop:readout-filter-bound} and the hidden-label threshold
test give the following bounds, where $F_{\max}$ is the largest readout
formula-class size:
\[
\begin{aligned}
  \Pp(F_{\mathrm{ro}})
  &\le K_{\mathrm{ro}}(F_{\max}+1)
      e^{-2m_{\mathrm{ro}}\gamma_{\mathrm{ro}}^2},\\
  \Pp(F_{\mathrm{hid}})
  &\le \delta_{\mathrm{set}}
      +2K_{\mathrm{hid}}e^{-m_{\mathrm{hid}}\gamma_{\mathrm{hid}}^2},
\end{aligned}
\]
where Proposition~\ref{prop:setup-cost} supplies an attempted-action budget
for $\delta_{\mathrm{set}}$. Diagnostic localization and local retesting obey
\[
\begin{aligned}
  \Pp(F_{\mathrm{diag}})
  &\le ne^{-m_{\mathrm{diag}}\Delta^2/2},\\
  \Pp(F_{\mathrm{upd}})
  &\le\sum_{j\in B}|\mathcal N_j|e^{-r_j\eta_j^2/2}.
\end{aligned}
\]
For example, the sufficient allocations
\[
\begin{aligned}
  m_g&\ge\gamma_g^{-2}\log\frac{2K_g}{\delta_g},
  &&g\in\{\mathrm{tar},\mathrm{ctx},\mathrm{time}\},\\
  m_{\mathrm{ro}}&\ge\frac{1}{2\gamma_{\mathrm{ro}}^2}
  \log\frac{K_{\mathrm{ro}}(F_{\max}+1)}{\delta_{\mathrm{ro}}},\\
  m_{\mathrm{hid}}&\ge\frac{1}{\gamma_{\mathrm{hid}}^2}
  \log\frac{2K_{\mathrm{hid}}}{\delta_{\mathrm{label}}},\\
  m_{\mathrm{diag}}&\ge\frac{2}{\Delta^2}
  \log\frac{n}{\delta_{\mathrm{diag}}},\\
  r_j&\ge\frac{2}{\eta_j^2}
  \log\frac{|B|\,|\mathcal N_j|}{\delta_{\mathrm{upd}}}
\end{aligned}
\]
make the corresponding terms no larger than their assigned failure budgets.
Set
\[
  \delta_{\mathrm{src}}
  =\delta_{\mathrm{tar}}+\delta_{\mathrm{ctx}}+\delta_{\mathrm{ro}}
   +\delta_{\mathrm{time}}+\delta_{\mathrm{hid}},
  \qquad
  \delta_{\mathrm{hid}}=\delta_{\mathrm{set}}+\delta_{\mathrm{label}}.
\]
If
$\delta_{\mathrm{src}}+\delta_\phi+
\delta_{\mathrm{diag}}+\delta_{\mathrm{upd}}\le\delta$, substitution into
Theorem~\ref{thm:end-to-end-risk-main} yields
$\Pp(L_{\mathrm{adapt}}>0)\le\delta$.

\section{Experimental Protocols and Results}
\label{app:experimental-protocols}

\subsection{Fixed-Memory Evaluation}
\label{app:finite-scm-results}

During interaction, each method may use its allowed internal state to choose
actions. Learning then stops. The evaluator queries only the frozen answer
map $T_{\widehat\mu}$ on held-out probes and compares it with the SCM answer
$G_M$. Ground-truth mechanisms are used to construct probes and scores after
the run; they are never supplied to the writer during interaction.

Exact mechanism accuracy uses tuple matches. The finite discovery and language
tables report target-edge precision, recall, and F1 from
Eq.~\eqref{eq:target-edge-projection}, while context, value, delay, and readout
coordinates are probed separately. The FP column counts
$|E(\widehat\mu)\setminus E(\mu^\star)|$. Readout false positives are the
subset whose direct-target field is a designated sensor or deterministic
proxy. Hidden recall counts true hidden-gated mechanisms
whose context type is retained. AdaptScore equals one only when all shifted
entries are corrected and all stable entries are preserved.

\subsection{Baselines}

The reward learner stores effects associated with high-return actions. Passive
correlation stores observed changes without action contrasts. Conditional
discovery thresholds conditional transition differences; its readout-oracle
variant removes known readouts before scoring. Tuple lift converts discovered
dependencies into the full output schema, and its oracle variant receives the
same readout information. The latent world model and neural transition probe
fit source transitions and are decoded only after their states are frozen.
TD-MPC2 uses the released MT30 checkpoint; linear and nonlinear response
decoders are fitted to source intervention pairs and frozen before the actuator
remapping. Its global and selective updates receive identical target pairs.
The ranking predictor minimizes pairwise intervention-ranking loss. Ungated
writers admit every proposed tuple; gated variants use the same proposals but
apply Eqs.~\eqref{eq:admitted-set} and \eqref{eq:memory-write}. These controls
separate objective and interface effects from output-format effects.

\begin{table}[tbp]
\centering
\small
\setlength{\tabcolsep}{3.5pt}
\begin{tabular}{L{0.34\linewidth}rrrrr}
\toprule
Method & Target F1 & FP & Readout FP & Hidden recall & Hidden FP \\
\midrule
Causal Core context search & $0.793$ & $1.063$ & $0.000$ & $0.458$ & $0.000$ \\
Conditional discovery, tuned & $0.640$ & $8.438$ & $6.813$ & $0.417$ & $3.250$ \\
Conditional discovery, readout oracle & $0.763$ & $1.938$ & $0.000$ & $0.417$ & $1.500$ \\
Tuple lift, tuned & $0.650$ & $8.000$ & $7.438$ & $0.479$ & $10.875$ \\
Tuple lift, readout oracle & $0.789$ & $1.313$ & $0.000$ & $0.500$ & $6.188$ \\
Latent world model, tuned & $0.239$ & $48.688$ & $16.250$ & $0.250$ & $46.750$ \\
Latent world model, readout oracle & $0.319$ & $43.563$ & $0.000$ & $0.250$ & $39.875$ \\
Transition probe, tuned & $0.396$ & $38.625$ & $13.250$ & $0.313$ & $28.063$ \\
Transition probe, readout oracle & $0.479$ & $25.375$ & $0.000$ & $0.313$ & $21.188$ \\
\bottomrule
\end{tabular}
\caption{Matched finite-SCM baselines.}
\label{tab:matched-baselines}
\end{table}

All rows in Table~\ref{tab:matched-baselines} receive the same trajectories.
The latent world model attains next-bit accuracy $0.905$ (and $0.907$ with the
readout oracle), so its probe error is not explained by failed source
prediction. Treating each generated SCM family as the independent unit, the
mean target-edge F1 and family-clustered 95\% $t$ intervals are $0.793$ [$0.749,0.837$]
for Causal Core, $0.640$ [$0.606,0.673$] for conditional discovery,
$0.789$ [$0.740,0.839$] for oracle tuple lift, $0.239$ [$0.185,0.292$] for
latent modeling, and $0.396$ [$0.387,0.404$] for the transition probe. The
intervals are descriptive across the eight generated families; the exact
family-seed values are in the source-data archive.

\begin{table}[tbp]
\centering
\small
\setlength{\tabcolsep}{4pt}
\begin{tabular}{L{0.31\linewidth}L{0.17\linewidth}L{0.20\linewidth}L{0.22\linewidth}}
\toprule
Task & Ranking signal & Ranking predictor & Causal Core \\
\midrule
Noisy-hidden ranking & Source rank $0.521$ & Target F1 $0.000$; hidden $0.000$ & Target F1 $0.764$; hidden $0.672$ \\
Renamed adaptation & Source/target rank $0.865/0.500$ & AdaptScore $0.000$ & AdaptScore $1.000$ \\
Semantic-readout adaptation & Source/target rank $0.756/0.428$ & AdaptScore $0.000$ & AdaptScore $1.000$ \\
\bottomrule
\end{tabular}
\caption{Ranking and selective-adaptation results.}
\label{tab:ranking-adaptation}
\end{table}

\subsection{Continuous Metric Cells}

The continuous SCM has two physical state coordinates, a two-dimensional
action, a deterministic pressure readout, and a context-specific vent effect.
The evaluator fixes metric cells before training. A pulse is correct only when
its direct target, sign, and context cell match the reference answer.

\begin{table}[tbp]
\centering
\scriptsize
\setlength{\tabcolsep}{3.5pt}
\begin{tabular}{L{0.31\linewidth}rrrrrr}
\toprule
Method & Precision & Recall & F1 & FP & Readout FP & Context error \\
\midrule
Random correlation & $0.252$ & $1.000$ & $0.402$ & $17.85$ & $5.95$ & $0.00$ \\
Global linear regression & $0.500$ & $1.000$ & $0.667$ & $6.00$ & $4.00$ & $2.00$ \\
Metric Causal Core & $1.000$ & $1.000$ & $1.000$ & $0.00$ & $0.00$ & $0.00$ \\
\bottomrule
\end{tabular}
\caption{Continuous metric-cell results.}
\label{tab:continuous-results}
\end{table}

For noise levels $0.02,0.05,0.10,0.15$, metric Causal Core retains F1
$1.000$. The corresponding random-correlation F1 scores are
$0.402,0.321,0.265,0.252$; global linear scores are
$0.667,0.667,0.667,0.695$.

\subsection{Distributional World-Model Probes}
\label{app:world-model-probes}

The public TD-MPC2 MT30 checkpoint contains $5{,}241{,}066$ parameters and is
evaluated without weight updates. Its SHA-256 digest, checkpoint metadata, and
all per-run outputs are included with the source data. The tasks are
cheetah-run, walker-walk, and reacher-easy, with three state-sampling seeds per
task. States are collected under the checkpoint policy with Gaussian action
noise. The evaluator draws unit action directions and horizons
$h\in\{1,2,4\}$, then branches the MuJoCo physics state to compute
Eq.~\eqref{eq:distributional-response}. Each source decoder receives 900
training queries; source and target evaluation use 240 held-out queries each.

The linear decoder is ridge regression. The nonlinear decoder has two
256-unit GELU layers, uses a disjoint 15\% validation split, and stops after 24
epochs without improvement. Both receive the encoded state, the difference of
positive and negative latent rollouts, the action direction, and the horizon.
The normalized response error is
\[
 \operatorname{NRMSE}
 =\left(
 \frac{\sum_i\|\widehat\Delta_i-\Delta_i\|_2^2}
      {\sum_i\|\Delta_i\|_2^2}
 \right)^{1/2}.
\]
Effect-sign accuracy evaluates the sign at the largest-magnitude coordinate of
the true response.

For seed $s$, actuator $(s-1)\bmod d_a$ changes sign. Adaptation uses five
states per actuator. The selective update compares positive and negative
actuator hypotheses and changes an entry only when the alternative halves its
paired loss. The global update fits an unrestricted residual probe head to the
same target pairs. Thresholds from $0.45$ through $0.90$ all localize the same
single actuator in every run.

\begin{table}[tbp]
\centering
\small
\setlength{\tabcolsep}{3.5pt}
\begin{tabular}{L{0.23\linewidth}rrrrr}
\toprule
Task & Source NRMSE & Frozen sign & Global sign & Selective sign & Localized \\
\midrule
Cheetah run & $0.701$ & $0.028$ & $0.794$ & $0.975$ & $3/3$ \\
Walker walk & $0.841$ & $0.142$ & $0.786$ & $0.869$ & $3/3$ \\
Reacher easy & $0.343$ & $0.003$ & $0.500$ & $1.000$ & $3/3$ \\
\midrule
Mean & $0.629$ & $0.057$ & $0.694$ & $0.948$ & $9/9$ \\
\bottomrule
\end{tabular}
\caption{Distributional intervention responses.}
\label{tab:tdmpc2-responses}
\end{table}

On random target directions, NRMSE is $1.138$ for the frozen model, $1.548$ for
the global update, and $0.630$ for the selective update. On directions that
exclude the shifted actuator, the corresponding stable-response errors are
$0.673$, $1.460$, and $0.673$. The selective update therefore matches the
oracle actuator map at the reported precision without changing stable
responses. The fixed finite query support together with a bounded truncation
of Euclidean loss instantiates Eq.~\eqref{eq:metric-retention-risk};
untruncated NRMSE is reported to expose large errors. The experiment does not
infer an unrestricted latent SCM.

\subsection{Language-Model States and Proposal Policies}
\label{app:llm-evaluation}

The hidden-state experiment presents Qwen2.5-7B-Instruct with a controlled
intervention record followed by a binary direct-target or delay probe. Twelve
procedural families supply 504 balanced training probes; six disjoint families
supply 360 balanced test probes. Each family and condition receives a fresh
random permutation of neutral action and state tokens, so a decoder cannot
identify a mechanism from a recurring name. Every record contains two
controlled repetitions per action. Prompt length is 697 tokens on average and
812 at maximum, with no truncation.

The language model is frozen. At the answer position, hidden vectors are
extracted from layers $0,7,14,21,$ and $28$. A logistic decoder is fit at each
layer, and a two-layer multilayer perceptron and an RBF support-vector decoder
are fit to the last layer. The controls are query-only TF--IDF, the model's
calibrated Yes/No logit margin, and a paired-evidence rule. The gated mechanism
state admits a candidate only when the controlled pair supports its lag and
unrelated interventions do not support the same target. All decoders and
thresholds are frozen before evaluation on the six test families. Balanced
accuracy is reported for mixed-label probes; readout FPR is the acceptance
rate on synchronized-readout candidates. Standard errors in
Figure~\ref{fig:llm-hidden-retention} use the six test families as clusters.

\begin{table}[tbp]
\centering
\scriptsize
\setlength{\tabcolsep}{4pt}
\begin{tabular}{L{0.34\linewidth}rrrr}
\toprule
State or decoder & Source BA & Renamed BA & Shifted-delay BA & Readout FPR \\
\midrule
Query-only TF--IDF & $0.736$ & $0.611$ & $0.000$ & $0.611$ \\
Qwen Yes/No logits & $0.653$ & $0.583$ & $0.417$ & $0.000$ \\
Qwen last-layer linear & $0.958$ & $0.847$ & $0.583$ & $0.000$ \\
Qwen last-layer RBF & $0.944$ & $0.875$ & $0.250$ & $0.111$ \\
Paired evidence & $1.000$ & $1.000$ & $1.000$ & $0.944$ \\
Gated mechanism state & $1.000$ & $1.000$ & $1.000$ & $0.056$ \\
\bottomrule
\end{tabular}
\caption{Frozen language-model state probes.}
\label{tab:llm-hidden-results}
\end{table}

The source score of the last-layer linear decoder shows that substantial
mechanism information is present in the frozen hidden state. Its
shifted-delay score of $0.583$, together with the RBF decoder's score of
$0.250$, shows that source decodability does not yield a stable answer under a
held-out mechanism change. Paired evidence recovers the changed delay but
accepts $94.4\%$ of synchronized readouts. The gate retains its perfect
shifted-delay score while reducing that false-positive rate to $5.6\%$.

Qwen2.5-7B-Instruct is used only to propose the next legal intervention from
the current transcript. Runs use deterministic decoding, 100 interaction
steps, and three seeds. The action is executed in the same complex
noisy-hidden family used by the non-language comparison. Explanations and
reward predictions receive no probe credit. Invalid outputs are mapped by the
same deterministic action parser in every language condition.

\begin{table}[tbp]
\centering
\scriptsize
\setlength{\tabcolsep}{3pt}
\begin{tabular}{L{0.34\linewidth}rrrrrr}
\toprule
Agent & Precision & Recall & Target F1 & FP & Readout FP & Hidden recall \\
\midrule
Ungated Qwen explorer & $0.127$ & $0.430$ & $0.196$ & $92.3$ & $37.0$ & $0.00$ \\
Qwen with causal prompt & $0.109$ & $0.323$ & $0.163$ & $82.3$ & $32.0$ & $0.00$ \\
Qwen + context-search gate & $0.767$ & $0.613$ & $0.678$ & $6.0$ & $0.0$ & $0.17$ \\
Qwen + control-planner gate & $0.901$ & $0.613$ & $0.727$ & $2.0$ & $0.0$ & $0.08$ \\
Pure context-search planner & $0.823$ & $0.581$ & $0.680$ & $4.0$ & $0.0$ & $0.17$ \\
Pure control planner & $0.790$ & $0.634$ & $0.697$ & $5.7$ & $0.0$ & $0.17$ \\
\bottomrule
\end{tabular}
\caption{Language-model proposal policies.}
\label{tab:llm-proposal-results}
\end{table}

The causal prompt does not remove readout contamination. The gated variants
use the same language model as a proposal source but apply the evidence gate
after observing the transition. Their advantage therefore comes from the
memory update, not from treating generated explanations as causal labels.

\subsection{Candidate Budgets and Implementation}

Candidate schedules cap computation but do not enter the probe metric. Small,
default, and wide schedules yield the same target-edge F1 at 80 and 160 steps
($0.698$ and $0.763$), zero readout false positives, and hidden recall rising
from $0.125$ to $0.625$. All thresholds, target budgets, metric cells, saved
MuJoCo states, language prompts, and parsers are fixed before held-out scoring;
malformed language proposals count as failures.

Finite-SCM actions are interventions. The writer estimates
Eq.~\eqref{eq:target-gate} from action-conditioned lift and repeated context
visits. Because these trajectories are adaptive, the iid bounds in
Appendix~\ref{app:contrast-gates} are not reported as confidence intervals for
the empirical tables. Lemma~\ref{lem:adaptive-contrast} covers randomized
adaptive logging; the continuous protocols use the paired design directly.
Candidate caps determine which contrasts are attempted, not how frozen probes
are scored. Numerical tables and raw JSON files are included as source data.

\bibliography{references}

\end{document}